\documentclass[10pt,journal,compsoc]{IEEEtran}
\usepackage{amsmath}
\usepackage{amssymb}
\usepackage{epsfig}
\usepackage{graphicx}
\usepackage[utf8]{inputenc}
\usepackage[T1]{fontenc}
\usepackage{url}
\usepackage{microtype}
\usepackage{lipsum}
\usepackage{graphicx}
\usepackage{booktabs}
\usepackage{amsfonts}
\usepackage{verbatim}
\usepackage{makecell}
\usepackage{multirow}

\usepackage{amsthm}
\newtheorem{theorem}{Theorem}
\newtheorem{proposition}{Proposition}
\newtheorem{assumption}{Assumption}

\newtheorem{definition}{Definition}

\usepackage{bm}
\usepackage{caption}
\usepackage{booktabs,subcaption,amsfonts,dcolumn}
\usepackage[switch]{lineno}
\usepackage[pagebackref=false,breaklinks=false,colorlinks,citecolor=blue,bookmarks=false]{hyperref}

\usepackage{dsfont}
\usepackage{pifont}
\DeclareMathOperator*{\argmin}{argmin}
\DeclareMathOperator*{\argmax}{argmax}

\usepackage{xspace}

\makeatletter
\DeclareRobustCommand\onedot{\futurelet\@let@token\@onedot}
\def\@onedot{\ifx\@let@token.\else.\null\fi\xspace}

\def\eg{\emph{e.g}\onedot} 
\def\ie{\emph{i.e}\onedot}

\def\etal{\emph{et al}\onedot}
\makeatother

\usepackage{color}
\usepackage[dvipsnames]{xcolor}
\usepackage{colortbl}

\definecolor{Gray}{gray}{0.9}
\definecolor{LightBlue}{RGB}{236,244,249}

\definecolor{color1}{HTML}{ECF4F9}
\definecolor{color2}{HTML}{FFF1E0}
\definecolor{color3}{HTML}{ECF4E9}
\newcommand{\CC}[1]{\cellcolor{color1}}
\newcommand{\RC}[1]{\rowcolor{color1}}
\newcommand{\RD}[1]{\rowcolor{color2}}
\newcommand{\RE}[1]{\rowcolor{color3}}

\newcommand{\cmark}{\color{green}\ding{51}}%
\newcommand{\xmark}{\color{red}\ding{55}}%

\usepackage{algorithm}
\usepackage{algpseudocode}
\algrenewcommand{\algorithmicrequire}{\textbf{Input:}}
\algrenewcommand{\algorithmicensure}{\textbf{Output:}}

\definecolor{simgcd}{HTML}{3498DB}
\newcommand{\simgcd}[1]{\textbf{{\color{simgcd}#1}}}
\definecolor{protogcd}{HTML}{E67E22}
\newcommand{\protogcd}[1]{\textbf{{\color{protogcd}#1}}}

\begin{document}

\title{ProtoGCD: Unified and Unbiased Prototype Learning for Generalized Category Discovery}

\author{Shijie~Ma,
	Fei~Zhu,
	Xu-Yao~Zhang,~\IEEEmembership{Senior Member,~IEEE},
	and~Cheng-Lin~Liu,~\IEEEmembership{Fellow,~IEEE}
	\IEEEcompsocitemizethanks{
    \IEEEcompsocthanksitem This work has been supported by the National Science and Technology Major Project (2022ZD0116500),  National Natural Science Foundation of China (62222609, 62076236), CAS Project for Young Scientists in Basic Research (YSBR-083), Key Research Program of Frontier Sciences of CAS (ZDBS-LY-7004), and the InnoHK program. (Corresponding author: Xu-Yao Zhang.)
    \IEEEcompsocthanksitem Shijie Ma, Xu-Yao Zhang and Cheng-Lin Liu are with the State Key Laboratory of Multimodal Artificial Intelligence Systems, Institute of Automation, Chinese Academy of Sciences, 95 Zhongguancun East Road, Beijing 100190, P.R. China, and also with the School of Artificial Intelligence, University of Chinese Academy of Sciences, Beijing 100049, China.
    Email: mashijie2021@ia.ac.cn, \{xyz, liucl\}@nlpr.ia.ac.cn.
    \IEEEcompsocthanksitem Fei Zhu is with the Centre for Artificial Intelligence and Robotics, Hong Kong Institute of Science and Innovation, Chinese Academy of Sciences, Hong Kong 999077, P.R. China.
    Email: zhfei2018@gmail.com.
    \IEEEcompsocthanksitem Code is available at {https://github.com/mashijie1028/ProtoGCD}.
    }
}

\markboth{IEEE TRANSACTIONS ON PATTERN ANALYSIS AND MACHINE INTELLIGENCE}
{Ma \MakeLowercase{\textit{et al.}}: ProtoGCD: Unified and Unbiased Prototype Learning for Generalized Category Discovery}

\IEEEtitleabstractindextext{
\begin{abstract}
Generalized category discovery (GCD) is a pragmatic but underexplored problem, which requires models to automatically cluster and discover novel categories by leveraging the labeled samples from old classes. The challenge is that unlabeled data contain both old and new classes. Early works leveraging pseudo-labeling with parametric classifiers handle old and new classes separately, which brings about imbalanced accuracy between them. Recent methods employing contrastive learning neglect potential positives and are decoupled from the clustering objective, leading to biased representations and sub-optimal results. To address these issues, we introduce a unified and unbiased prototype learning framework, namely ProtoGCD, wherein old and new classes are modeled with joint prototypes and unified learning objectives, {enabling unified modeling between old and new classes}. Specifically, we propose a dual-level adaptive pseudo-labeling mechanism to mitigate confirmation bias, together with two regularization terms to collectively help learn more suitable representations for GCD. Moreover, for practical considerations, we devise a criterion to estimate the number of new classes.
Furthermore, we extend ProtoGCD to detect unseen outliers, achieving task-level unification.
Comprehensive experiments show that ProtoGCD achieves state-of-the-art performance on both generic and fine-grained datasets.
\end{abstract}

}

\maketitle
\IEEEdisplaynontitleabstractindextext
\IEEEpeerreviewmaketitle
\IEEEraisesectionheading{\section{Introduction}\label{sec:introduction}}

\IEEEPARstart{H}{umans} are capable of discovering and acquiring novel concepts based on what they have learned~\cite{troisemaine2023novel,9464163,zhu2024open}. Consider that a kid has been taught to recognize some species (\eg, ``cat'', ``panda'', ``car'') and gradually grasp some general knowledge, \ie, \emph{what constitutes a class}. Then, the kid could cluster some ``tiger'' images together and regard them as a novel category even without learning them before, as shown in Fig.~\ref{fig:gcd-setting}. Accordingly, it is important to empower such ability to deep learning and make it more applicable in the \emph{open-world}~\cite{zhu2024open,9040673,salehi2022a,ZHAO2024104059}, where samples from new classes might emerge and models are expected to discover them by transferring the knowledge from old classes.

Recently, novel category discovery (NCD)~\cite{troisemaine2023novel,9464163,Han2020Automatically,zhong2021openmix,Zhong_2021_CVPR,zhao21novel,Li_2023_CVPR} has been introduced to solve the aforementioned problem. Formally, NCD aims to automatically cluster the unlabeled novel classes by leveraging the knowledge learned from old classes in the labeled dataset. It assumes that unlabeled data exclusively comprises samples from novel categories, which often fails to hold in reality. By relaxing such a strong assumption, Vaze~\etal~\cite{vaze2022gcd} extended NCD to a more pragmatic setting, called generalized category discovery (GCD). In GCD, images from unlabeled data could contain both old and new classes.

\begin{figure}[!t]
    \centering
    \includegraphics[width=.95\linewidth]{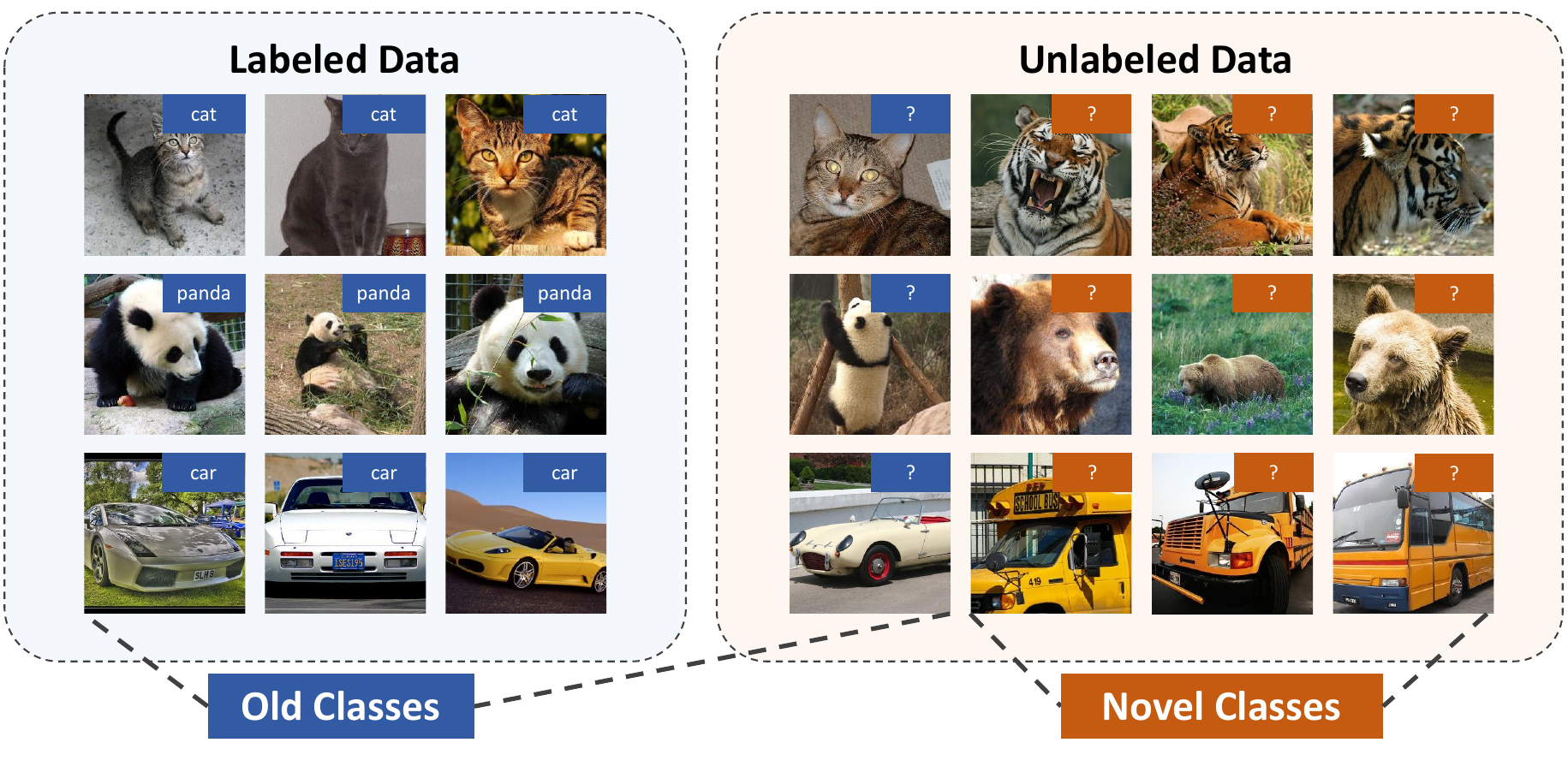}
    \vspace{-5pt}
    \caption{Generalized category discovery. Given a dataset with labeled data from old classes and unlabeled data from both old and novel categories. The objective is to classify old classes and cluster new categories in the unlabeled data.}
    \label{fig:gcd-setting}
    \vspace{-10pt}
\end{figure}

\begin{figure*}[!t]
    \centering
    \includegraphics[width=.95\linewidth]{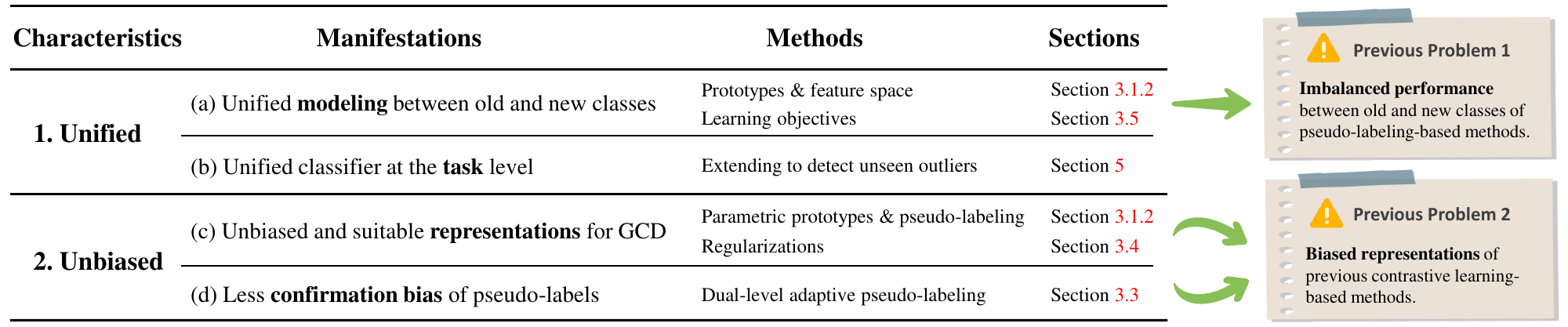}
    \vspace{-5pt}
    \caption{The \textbf{unified} and \textbf{unbiased} characteristics of ProtoGCD, which contribute to addressing the issues of prior methods.}
    \label{fig:two-charasteristics}
\end{figure*}

In this paper, we tackle the task of GCD~\cite{vaze2022gcd,fei2022xcon,Zhao_2023_ICCV,pu2023dynamic} as illustrated in Fig.~\ref{fig:gcd-setting}, which is a challenging \textit{open-world}~\cite{zhu2024open,salehi2022a,ZHAO2024104059} setting in that models need to simultaneously discover novel categories and recognize old classes coexisting in the unlabeled data. Pioneer works~\cite{vaze2022gcd,fei2022xcon} resort to the supervised~\cite{NEURIPS2020_d89a66c7} and unsupervised contrastive learning~\cite{pmlr-v119-chen20j} on labeled data and unlabeled data, respectively. And non-parametric semi-supervised K-means~\cite{vaze2022gcd,arthur2007k} is employed upon the learned features for clustering. However, contrastive learning alone ignores underlying positives and is susceptible to \textit{class collision}~\cite{Zheng_2021_ICCV}. Furthermore, pure contrastive learning is essentially decoupled with the clustering objective of GCD, leading to biased representations and sub-optimal performance.
Another line of works~\cite{9464163,Fini_2021_ICCV} use pseudo-labels and handle old and novel classes with separate classification heads and learning objectives. These methods tend to be biased toward old classes~\cite{vaze2022gcd}, and bring about imbalanced accuracies between old and new classes. The problems of preceding methods are summarized in Fig.~\ref{fig:two-charasteristics}.

To solve the issues above, we propose a unified and unbiased \textbf{Proto}type Learning framework for \textbf{G}eneralized \textbf{C}ategory \textbf{D}iscovery (ProtoGCD), which handles old and novel categories jointly in a shared feature space with the same set of learnable prototypes. {There are two key insights to solve the issues of prior methods:} (1) The first is the unification between old and new classes (Fig.~\ref{fig:two-charasteristics} (a)). We model old and new classes with a joint classifier and unified learning objectives, which helps alleviate the imbalanced performance of prior parametric methods~\cite{9464163,Fini_2021_ICCV}, {as shown in Fig.~\ref{fig:unification} (a) and (b).} (2) Secondly, the model is equipped with a parametric prototypical classifier and self-trained with pseudo-labels, which aligns more closely with the clustering objective and learns more suitable representations for GCD (Fig.~\ref{fig:two-charasteristics} (c)) than contrastive learning-based methods~\cite{vaze2022gcd,fei2022xcon,Zhao_2023_ICCV,pu2023dynamic}. Specifically, considering the challenging annotation conditions in GCD, we propose a dual-level adaptive pseudo-labeling (DAPL) mechanism. The model adaptively adjusts both the type and proportion of pseudo-labels assigned to unlabeled samples, according to the samples' confidence and the model's performance. DAPL ensures efficient and stable self-training while effectively circumventing \textit{confirmation bias}~\cite{pseudoLabel2019} (Fig.~\ref{fig:two-charasteristics} (d)). Additionally, two regularization terms (Fig.~\ref{fig:two-charasteristics} (c)) are further proposed to avoid trivial solutions of clustering and learn better features. Besides, we propose a novel criterion that simultaneously considers the feature space and the classification performance of old classes to precisely estimate the number of new classes, enabling our method to manage the more challenging situation where the number of novel categories is unknown. As a whole, the feature extractor and learnable prototypes are trained together in an end-to-end manner, making ProtoGCD learn efficiently and achieve remarkable performance.

Furthermore, beyond GCD, we explore the unification at the task level (Fig.~\ref{fig:two-charasteristics} (b)), and extend ProtoGCD to detect unseen outliers. As in Fig.~\ref{fig:unification} (c), ProtoGCD could classify both the old classes and the previously discovered new classes, as well as detect unseen outliers, which makes it a potentially unified open-world classifier.

Our main contributions are summarized as follows:
\begin{itemize}
    \item We propose ProtoGCD, a unified and unbiased framework for the task of GCD, which effectively addresses the issues of imbalanced performance and biased representations in previous methods.
    \item The unified modeling helps ProtoGCD achieve balanced accuracy between old and new classes, and we propose dual-level adaptive pseudo-labeling and regularizations to learn unbiased representations.
    \item We devise \emph{Prototype Score} to estimate the number of novel classes, making our method more practical.
    \item At the task level, we extend ProtoGCD to detect outliers from unseen classes, and achieve the unification of multiple tasks. 
    \item Experiments on generic and fine-grained datasets show that ProtoGCD outperforms previous state-of-the-art methods by a large margin and \emph{Prototype Score} obtains more accurate class number estimation.
\end{itemize}

The remainder is organized as follows: Section~\ref{sec:related-work} shows related works. Section~\ref{sec:method-protogcd} introduces the proposed ProtoGCD. Section~\ref{sec:estimate-number} presents the class number estimation algorithm and Section~\ref{sec:extend-to-ood} extends ProtoGCD to detect unseen outliers. Section~\ref{sec:experiments} provides comprehensive experiments and Section~\ref{sec:conclusion} concludes the paper and outlines future works.

\begin{figure*}[!t]
    \centering
    \includegraphics[width=.85\linewidth]{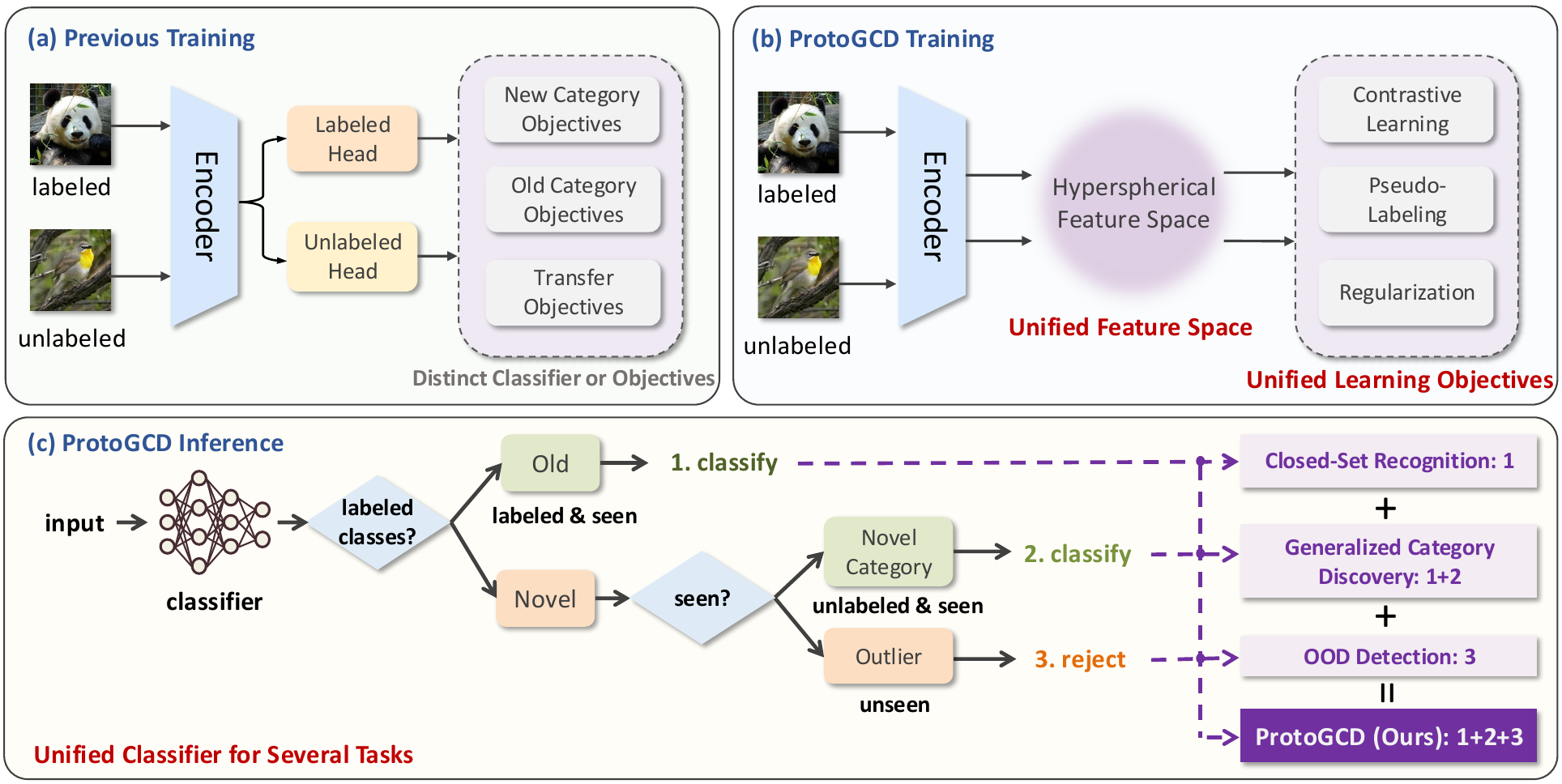}
    \vspace{-5pt}
    \caption{Unified prototype learning framework. (a) Previous GCD methods~\cite{Han2020Automatically,Fini_2021_ICCV,cao2022openworld} with parametric classifiers employ distinct classification heads or training objectives for old and new classes, while (b) ProtoGCD models old and new classes in a shared feature space with a unified set of prototypes (\ie, classifier) and adopts unified learning objectives across old and new classes. (c) During inference, ProtoGCD could classify both the old and the newly discovered classes. Moreover, it could also be extended to reject unseen outliers, which makes ProtoGCD a general-purpose open-world classifier.}
    \label{fig:unification}
\end{figure*}

\section{Related Works} \label{sec:related-work}

\subsection{Novel Category Discovery}
Novel category discovery (NCD) is initially formulated as a deep transfer clustering~\cite{Han_2019_ICCV} problem. The core spirit is to leverage the knowledge learned from labeled classes to cluster unlabeled data from novel categories. AutoNovel~\cite{9464163,Han2020Automatically} is a seminal work involving three steps. Models are firstly pre-trained via self-supervision and then fine-tuned on labeled datasets. Finally, models transfer knowledge from labeled data to unlabeled data through rank statistics~\cite{Han2020Automatically,zhao21novel}. UNO~\cite{Fini_2021_ICCV} proposed a unified objective and assigned pseudo-labels with swapped prediction~\cite{caron2020unsupervised}, while OpenMix~\cite{zhong2021openmix} and NCL~\cite{Zhong_2021_CVPR} further explored the relationship between labeled and unlabeled data.

\subsection{Generalized Category Discovery}
Vaze \etal~\cite{vaze2022gcd} relaxed the assumption in NCD that all unlabeled data comes from novel classes and formalized a more pragmatic task called generalized category discovery (GCD).
We categorize existing methods into two groups. (1) \emph{Contrastive learning-based methods with non-parametric classifiers}. Pioneering works~\cite{vaze2022gcd,fei2022xcon} employed contrastive learning followed by non-parametric semi-supervised K-means clustering~\cite{vaze2022gcd,arthur2007k}. Subsequent works explore more underlying relationships. Zhao \etal~\cite{Zhao_2023_ICCV} extended prototypical contrastive learning~\cite{li2021prototypical} to an EM-like learning framework. Pu \etal~\cite{pu2023dynamic} proposed dynamic conceptional contrastive learning, which alternates between conception estimation and conceptional representation learning. In these methods, feature representation learning is decoupled with and not optimal for subsequent clustering. (2) \emph{Pseudo-labeling-based methods with parametric-classifiers}, like adapted methods from NCD~\cite{9464163,Fini_2021_ICCV}. These methods implement separate classification heads on old and new classes, leading to imbalanced performance, and the predictions are easily biased to old classes.
More recently, some works enhance the performance of GCD by exploiting instance-wise neighbors~\cite{yang2024learning} and complementary textual modality~\cite{zheng2024textual}. While others extend GCD to more learning paradigms~\cite{ma2024active,pu2024federated} and scenarios~\cite{zheng2024prototypical,liu2024novel}.
To address the respective problems of the two types of methods, we propose to unify old and novel classes' modeling with learnable prototypes and devise a proper pseudo-labeling mechanism to circumvent \textit{confirmation bias}.
Moreover, regarding that most methods assume the number of novel categories is known \textit{a-prior}, only a few~\cite{9464163,vaze2022gcd,Zhao_2023_ICCV} tackle the estimation issue. We also propose to estimate the class number precisely to make GCD more applicable.
Here, we summarize the differences between ProtoGCD and prior works. (1) Compared with non-parametric methods like~\cite{vaze2022gcd,Zhao_2023_ICCV} with contrastive learning, ProtoGCD explicitly learns a parametric classifier with discriminative self-training. (2) Compared with the recent parametric-based SimGCD~\cite{wen2023parametric}, ProtoGCD incorporates generative modeling, and the prototypes represent class-wise distributions, so ProtoGCD could be viewed as a hybrid model, while SimGCD is a purely discriminative model. Considering the characteristics of the GCD task, we further incorporate dual-level adaptivity into pseudo-labeling and propose separation regularization.

\subsection{Out-of-Distribution Detection}
Out-of-distribution (OOD) detection~\cite{hendrycks2016baseline,ma2025towards,yang2020convolutional,chen2021adversarial,huang2022class} aims to classify samples from known classes and reject unseen samples outside of the training classes. Conventionally, each sample is assigned a score. If the score is higher than a predefined threshold, then it is recognized as in-distribution (ID) and classified into one of the known classes, or detected as OOD and rejected.
Post-hoc methods aim to devise score functions~\cite{hendrycks2016baseline,liu2020energy,pmlr-v162-hendrycks22a} to increase the separability between ID and OOD instances without training the models.
Several works resort to self-supervised learning~\cite{hendrycks2019using,tack2020csi} and logit normalization~\cite{wei2022mitigating} to train models that inherently excel at OOD rejection. Others explicitly employ auxiliary outliers~\cite{hendrycks2018deep}. OOD detection only requires rejecting OOD samples without any further clustering on them.

\begin{figure*}[!t]
    \centering
    \includegraphics[width=.95\textwidth]{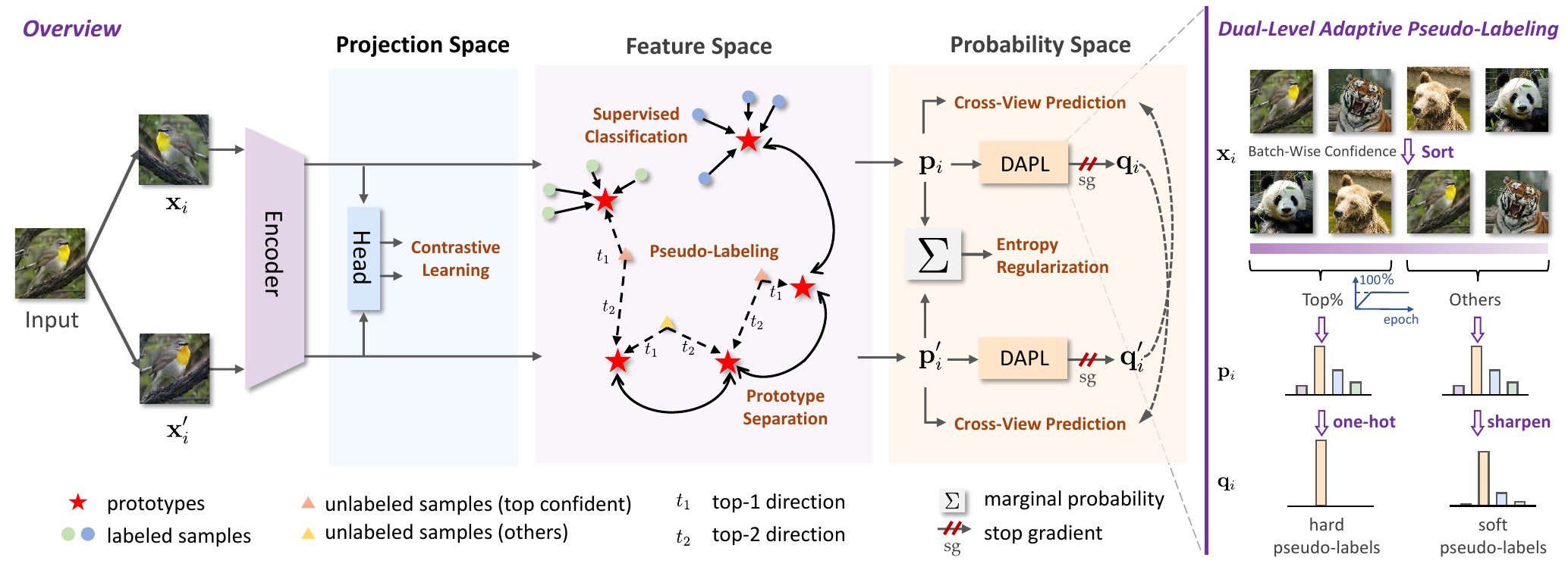}
    \vspace{-5pt}
    \caption{The proposed method ProtoGCD. Left: Overview of ProtoGCD. The blue, purple and orange backgrounds indicate the projection, feature and probability space, respectively. The yellow font represents learning objectives. Right: Dual-Level Adaptive Pseudo-Labeling (DAPL). We adaptively assign hard pseudo-labels to top $r\%$ samples by confidence while soft ones for the others, and the ratio $r\%$ adaptively ramps up (blue font). ProtoGCD could be trained end-to-end.}
    \label{fig:protogcd-pipeline}
\end{figure*}

\section{The Proposed Method: ProtoGCD} \label{sec:method-protogcd}

\emph{Motivation and Overview.}
As depicted in Fig.~\ref{fig:two-charasteristics}, we aim to address the issues of imbalanced performance and biased representations of prior methods. To maintain the balance between old and new classes, we propose to employ a unified prototypical classifier and feature space for them (Section~\ref{subsubsec:modeling-protogcd}). To acquire unbiased and suitable representations for GCD, we utilize contrastive learning for basic representations (Section~\ref{subsec:contrastive-learning}). More importantly, we propose an adaptive pseudo-labeling mechanism that dynamically adapts the types and proportions of the pseudo-labels, considering the \textit{samples' confidence} and the \textit{model's performance} (Section~\ref{subsec:pseudo-label}), which helps mitigate confirmation bias. Furthermore, two regularizations (Section~\ref{subsec:regularization}) help avoid trivial solutions and improve inter-class separation, thereby collectively refining the representations. Overall, ProtoGCD is an end-to-end training method, achieving unified learning objectives between old and new classes and aligning better with the clustering objectives of GCD (Section~\ref{subsec:overall-objective}). The overall pipeline of ProtoGCD is illustrated in Fig.~\ref{fig:protogcd-pipeline}.

\subsection{Preliminaries}
\label{subsec:preliminaries}

\subsubsection{Problem Formulation and Notation}
Formally, given a partially-labeled dataset $\mathcal{D}=\mathcal{D}_l\cup \mathcal{D}_u$, where $\mathcal{D}_l=\{(\mathbf{x}_i^l,y_i^l)\}_{i=1}^n\subset \mathcal{X}_l\times \mathcal{Y}_l$ denotes the labeled data from the old classes\footnote{In this paper, old classes and labeled classes are synonymous and both refer to the classes that appear in $\mathcal{D}_l$.}, \ie, $\mathcal{Y}_l=\mathcal{C}_{old}$, and $\mathcal{D}_u=\{(\mathbf{x}_j^u)\}_{j=1}^m\subset \mathcal{X}_u$ denotes the unlabeled data with its underlying label space $\mathcal{Y}_u$ consisting of both old classes
$\mathcal{C}_{old}$ and novel classes $\mathcal{C}_{new}$, \ie, $\mathcal{Y}_u=\mathcal{C}_{old}\cup\mathcal{C}_{new}$. The objective of GCD is to simultaneously cluster samples from $\mathcal{C}_{new}$ and classify samples from $\mathcal{C}_{old}$ in $\mathcal{D}_u$ with the prior knowledge in $\mathcal{D}_l$.
The number of old classes $K_{old}=|\mathcal{C}_{old}|$ can be obtained directly from $\mathcal{D}_l$, while the number of novel classes $K_{new}=|\mathcal{C}_{new}|$ is always known \emph{a-prior} in the literature~\cite{vaze2022gcd,fei2022xcon}. We also present an algorithm to estimate $K_{new}$ in Section~\ref{sec:estimate-number}. $K=K_{old}+K_{new}$ is the total number of classes. Let $\mathcal{E}(\cdot)$ denote the feature extractor, and $\phi(\cdot)$ is the projection head. $\mathbf{z}_i=\mathcal{E}(\mathbf{x}_i)$ is the $d$-dimensional feature representation of the $i$-th sample $\mathbf{x}_i$. $\mathbf{h}_i=\phi(\mathbf{z}_i)$ is in the $d_h$-dimensional projection space for contrastive learning.

\subsubsection{Principled Modeling of ProtoGCD} \label{subsubsec:modeling-protogcd}

\emph{Unified Feature Space and Prototypes (Classifiers).}
We adopt $\ell_2$-normalized $d$-dimensional hyperspherical feature space, which is compatible with contrastive learning~\cite{pmlr-v119-chen20j,wang2020understanding} and has less bias between classes. To realize unified modeling of old and new classes, we assign the same set of learnable prototypes $\mathcal{P}=\{\boldsymbol{\mu}_c\}_{c=1}^K$ where $K=K_{old}+K_{new}$, each class with one prototype $\boldsymbol{\mu}_c$. Both $\mathbf{z}_i$ and $\boldsymbol{\mu}_c$ are $\ell_2$-normalized in the feature space, and prototypes could be updated \textit{on-the-fly}.

\emph{Generative Modeling.}
ProtoGCD models both old and novel classes jointly in a shared $d$-dimensional hypersphere, \ie $(d-1)$-sphere, and each class-wise prototype $\boldsymbol{\mu}_c$ formalizes von Mises–Fisher (vMF) distribution~\cite{mardia2000directional} with the probability density of the $i$-th sample in the $c$-th class as follows:
\begin{equation}
\small
    p_\text{vMF}(\mathbf{z}_i;\boldsymbol{\mu}_c,\tau)=C_p(1/\tau)\exp(\boldsymbol{\mu}_c^\top \mathbf{z}_i/\tau),\ c=1,2,\cdots,K,
\end{equation}
where $\tau$ is the temperature and $\kappa=1/\tau$ is the \emph{concentration parameter}~\cite{mardia2000directional} of vMF with $C_p(\kappa)=\frac{\kappa^{p/2-1}}{(2\pi)^{p/2}I_{p/2-1}(\kappa)}$ and $I_v$ denotes the first kind of Bessel function at order $v$. The prototype $\boldsymbol{\mu}_c$ is \emph{mean direction} in vMF. Then we could draw the posterior probability of sample $\mathbf{x}_i$ belonging to class $k$:
\begin{equation}
\small
    \begin{aligned}
        p(y=k|\mathbf{z}_i,\tau) = \frac{p_\text{vMF}(\mathbf{z}_i;\boldsymbol{\mu}_k,\tau)}{\sum_{c=1}^K p_\text{vMF}(\mathbf{z}_i;\boldsymbol{\mu}_c,\tau)} = \frac{\exp(\boldsymbol{\mu}_k^\top \mathbf{z}_i/\tau)}{\sum_{c=1}^K \exp(\boldsymbol{\mu}_c^\top \mathbf{z}_i/\tau)}.
    \end{aligned}
    \label{eq:posterior-prob}
\end{equation}
In Eq.~\eqref{eq:posterior-prob}, logits are computed via cosine similarity between features and class-wise prototypes, and the posterior probability prediction vector $\mathbf{p}(\mathbf{z}_i,\tau)\in\mathbb{R}^K$:
\begin{equation}
    \mathbf{p}(\mathbf{z}_i,\tau)=\big(p(y=1|\mathbf{z}_i,\tau),\cdots,p(y=k|\mathbf{z}_K,\tau)\big).
    \label{eq:posterior-vector}
\end{equation}
The generative modeling with prototypes is more suitable to the \textit{open-world} and reduces open-space risk~\cite{9040673} as validated in~\cite{yang2020convolutional,chen2021adversarial,huang2022class}. In this paper, we generalize prototype learning to the more pragmatic setting of GCD, where we model unlabeled new classes with prototypes as well.

\subsection{Contrastive Learning} \label{subsec:contrastive-learning}
To maintain fundamental representations, we employ supervised contrastive learning~\cite{NEURIPS2020_d89a66c7} on $\mathcal{D}_l$ and unsupervised contrastive learning (\ie, self-supervised contrastive learning named SimCLR)~\cite{pmlr-v119-chen20j} on $\mathcal{D}_l\cup\mathcal{D}_u$ respectively, within the projection space, following the convention in the literature~\cite{vaze2022gcd,fei2022xcon}. Specifically, given two views (random augmentations) of the input $\mathbf{x}_i$ and $\mathbf{x}_i^\prime$ in a mini-batch $\mathcal{B}$, the unsupervised contrastive learning loss:
\begin{equation}
    \mathcal{L}_\text{con}^u=\frac{1}{|\mathcal{B}|}\sum_{i\in\mathcal{B}}-\log \frac{\exp(\mathbf{h}_i^\top \mathbf{h}_i^\prime/\tau_c)}{\sum_j \mathds{1}_{[j\neq i]}\exp(\mathbf{h}_i^\top \mathbf{h}_j/\tau_c)},
    \label{eq:loss-unsup-con}
\end{equation}
where $\mathds{1}_{[\cdot]}$ denotes the indicator function and equals to $1$ when the condition is true else $0$, $\tau_c$ denotes the temperature in contrastive learning. The supervised contrastive learning~\cite{NEURIPS2020_d89a66c7} on labeled data in $\mathcal{B}$ is:
\begin{equation}
\small
    \mathcal{L}_\text{con}^l=\frac{1}{|\mathcal{B}_l|}\sum_{i\in\mathcal{B}_l}\frac{1}{|\mathcal{N}(i)|}\sum_{q\in\mathcal{N}(i)}-\log\frac{\exp(\mathbf{h}_i^\top \mathbf{h}_q/\tau_c)}{\sum_{j}\mathds{1}_{[j\neq i]}\exp(\mathbf{h}_i^\top \mathbf{h}_j/\tau_c)},
    \label{eq:loss-sup-con}
\end{equation}
where $\mathcal{B}_l$ denotes labeled subset of $\mathcal{B}$ and $\mathcal{N}(i)$ denotes positive samples with the same label as $\mathbf{x}_i$. Then, we combine them and draw the overall contrastive learning objective:
\begin{equation}
    \mathcal{L}_\text{con}=(1-\lambda_\text{sup})\mathcal{L}_\text{con}^u+\lambda_\text{sup}\mathcal{L}_\text{con}^l,
    \label{eq:loss-con}
\end{equation}
where $\lambda_\text{sup}$ is the weight of supervised component.

\subsection{Dual-Level Adaptive Pseudo-Labeling} \label{subsec:pseudo-label}
GCD is a semi-supervised setting but subject to more stringent labeling conditions, \ie, unlabeled data contain new classes. Parametric classifiers are supposed to consider both old and novel classes when assigning pseudo-labels. As a result, they are more susceptible to \textit{confirmation bias}. As for pseudo-labels, learning solely on hard pseudo-labels~\cite{pseudoLabel2019,lee2013pseudo}, \ie, one-hot targets, is prone to bias accumulation due to overconfidence in incorrect information, particularly during early training stages when the classifier is less performant. Conversely, relying entirely on soft pseudo-labels~\cite{xie2016unsupervised,nigam2000analyzing,caron2021emerging}, which are less confident, could hinder model training due to less informative targets. Therefore, it is essential to consider both types of pseudo-labels simultaneously.

Thus, we pose a question: \textit{What constitutes suitable pseudo-labels for GCD?} Based on the discussions above, the crux of this question is to choose the suitable type of pseudo-labels, \ie, soft or one-hot, for each sample and determine the ratio of the two types.
We provide two aspects to determine the pseudo-labels: (1) The confidence of samples. Samples' confidence varies based on their distribution in the feature space. Those close to the decision boundary exhibit lower confidence, and assigning overly confident hard pseudo-labels to these samples could introduce bias. (2) The model's capabilities. During early training phases, the model has relatively weak classification performance and is prone to confirmation bias, which is not readily corrected by the model itself. As training progresses, the model becomes stronger, facilitating the generation of high-quality pseudo-labels. Considering the above aspects, we propose a dual-level adaptive pseudo-labeling (DAPL) mechanism. The primary philosophy is to adaptively assign pseudo-labels to unlabeled data based on the \emph{samples' confidence} across different training samples and \emph{model's capability} across various training phases. In this way, ProtoGCD is capable of training models efficiently while preventing potential bias.

\subsubsection{Level-1: Adaptivity across Training Samples} \label{subsubsec:level-1}
We propose to assign pseudo-labels flexibly based on the confidence of samples, which could help mitigate bias from overconfident pseudo-labels while preventing slow training from overly ambiguous ones. Let $\boldsymbol{\mu}_{t_1(i)},\boldsymbol{\mu}_{t_2(i)}$ denote the top-1 and top-2 prototypes of sample $\mathbf{z}_i$, having the maximum and second maximum cosine similarities with $\mathbf{z}_i$ respectively:
\begin{equation}
    t_1(i) = \argmax_{c=1,2,\cdots,K}\boldsymbol{\mu}_c^\top \mathbf{z}_i,\qquad t_2(i) = \argmax_{\substack{c=1,2,\cdots,K \\ c\neq t_1(i)}}\boldsymbol{\mu}_c^\top \mathbf{z}_i,
\end{equation}
where $K=K_{old}+K_{new}$. Here, we define confidence, namely \emph{prototype confidence}, in Definition~\ref{def:confidence}.
\begin{definition}[Prototype Confidence of Each Sample]
\label{def:confidence}
    The prototype confidence of sample $\mathbf{z}_i$ is defined as the ratio of the exponential of the cosine similarity between $\mathbf{z}_i$ and the top-1 prototype and the one with the top-2 prototype:
    \begin{equation}
        \text{proto\_conf}(\mathbf{z}_i) = \exp(\boldsymbol{\mu}_{t_1(i)}^\top\mathbf{z}_i/\tau)/\exp(\boldsymbol{\mu}_{t_2(i)}^\top\mathbf{z}_i/\tau).
        \label{eq:def-conf}
    \end{equation}
\end{definition}
\noindent Intuitively, Eq.~\eqref{eq:def-conf} indicates that the closer $\mathbf{z}_i$ to the top-1 prototype $\boldsymbol{\mu}_{t_1(i)}$ compared with the top-2 prototype $\boldsymbol{\mu}_{t_2(i)}$, the higher the confidence of $\mathbf{z}_i$. \textit{Prototype confidence} only involves the two most similar prototypes, which is relatively more robust and stable with less noise than using all the prototypes for confidence estimation, \eg, maximum softmax probability~\cite{hendrycks2016baseline,guo2017calibration} (MSP), \ie, $\max_{k}p(y=k|\mathbf{z}_i,\tau)$, regarding that a large number of unlabeled samples could bring potential bias, especially in early training stages. Moreover, the range of \textit{prototype confidence} is broader than MSP, enhancing the distinctiveness among samples.

\emph{Assign Hard or Soft Pseudo-Labels Based on Confidence.}
The pseudo-label of each sample $\mathbf{q}(\mathbf{z}_i)$ is determined by its confidence. If a sample has high confidence, it might be far from the decision boundary and more likely belongs to class $\argmax_k p(y=k|\mathbf{z}_i,\tau)$, and we assign a one-hot pseudo-label, which accelerates training with more informative targets. If the confidence is low, hard pseudo-labels could easily bring erroneous information, so we choose soft labels instead. Concretely, hard or one-hot pseudo-labels are employed when confidence is above a certain threshold $\delta$, otherwise soft labels. $\mathbf{p}(\mathbf{z}_i,\tau_\text{base})$ is the predictive vector in Eq.~\eqref{eq:posterior-vector}, then the adaptive pseudo-label of the $i$-th sample is:
\begin{equation}
\small
    \mathbf{q}(\mathbf{z}_i)\in\mathbb{R}^K=\left\{
    \begin{aligned}
        & \text{one\_hot}\Big(\mathbf{p}(\mathbf{z}_i,\tau_\text{base})\Big),&\text{if}\ \emph{proto\_conf}(\mathbf{z}_i)\geq \delta, \\
        & \mathbf{p}(\mathbf{z}_i,\tau_\text{sharp}),&\text{if}\ \emph{proto\_conf}(\mathbf{z}_i)<\delta.
    \end{aligned}
    \right.
    \label{eq:adaptive-pl}
\end{equation}
Here, $\delta>1$, $\text{one\_hot}(\cdot)$ denotes the one-hot operation, where the output is one at the index of the maximum input value, and zero for other indices. $\tau_\text{base}$ and $\tau_\text{sharp}$ are temperature in the original prediction and pseudo-labels. Temperature controls the \emph{hardness/certainty} of the pseudo-labels, and lower $\tau$ indicates more certain pseudo-labels. In Eq.~\eqref{eq:adaptive-pl}, $\tau_\text{base}>\tau_\text{sharp}$, \ie, for samples with confidence less than $\delta$, we still assign sharpened pseudo-labels $\mathbf{p}(\mathbf{z}_i,\tau_\text{sharp})$ than the original prediction, which encourages the model to gradually make more certain predictions, this sharpening mechanism is of vital importance to steadily enhance the model, which is validated in Section~\ref{subsec:ablation}. The hard pseudo-label could be viewed as a special case of the soft one with $\tau\to 0$.

\subsubsection{Level-2: Adaptivity across Training Phases} \label{subsubsec:level-2}
From an orthogonal perspective, the model's capabilities vary during training. Initially, models are weak and tend to produce biased pseudo-labels, so more soft pseudo-labels are suggested. As training progresses, the model gradually learns to distinguish between different categories. As a consequence, we would place greater trust in its predictions and reduce the threshold $\delta$ in Eq.~\eqref{eq:adaptive-pl}. However, directly determining the threshold is non-trivial. Here, we propose a more reasonable approach, in which we set the proportion of unlabeled samples to which we assign hard labels. At epoch $e$, we present one-hot pseudo-labels to the top $r\%$ unlabeled samples with the highest confidence, while soft labels for the left. And $\delta$ could be implicitly expressed by the $\lfloor|\mathcal{D}_u|\times r/100\rfloor$-th highest confidence of all unlabeled samples. For the proportion of samples assigned with hard pseudo-labels, we adopt a linear ramp-up function:
\begin{equation}
    r(e)=\left\{
    \begin{aligned}
        & \frac{e}{e_\text{ramp}}\times 100\%,\quad & \text{if}\ 0\leq e\leq e_\text{ramp}, \\
        & 100\%,\quad & \text{if}\ e>e_\text{ramp},
    \end{aligned}
    \right.
    \label{eq:ratio-ramp-up}
\end{equation}
where $r(e)\in [0\%,100\%]$ is a function of training epochs. In practice, there is no need to explicitly compute $\delta$. At epoch $e$, we could select the top $r(e)\%$ of samples with the highest confidence and assign one-hot labels, while softly sharpened labels for the remaining $\big(100-r(e)\big)\%$ of samples. As in Eq.~\eqref{eq:ratio-ramp-up}, the ratio of hard pseudo-labels $r(e)$ grows linearly from the beginning to the $e_\text{ramp}$-th epoch, then all are hard pseudo-labels in later epochs.

\subsubsection{Cross-view Prediction with Pseudo-Labels} \label{subsubsec:cross-view}
We perform pseudo-labeling on all the training data, and propose to learn with cross-view prediction~\cite{caron2020unsupervised} as follows:
\begin{equation}
    \mathcal{L}_\text{dapl}=\frac{1}{2|\mathcal{B}|}\sum_{i\in\mathcal{B}}\Big(\ell(\mathbf{q}^\prime_i,\mathbf{p}_i)+\ell(\mathbf{q}_i,\mathbf{p}^\prime_i)\Big),
    \label{eq:loss-dapl}
\end{equation}
where $\ell(\mathbf{q}^\prime,\mathbf{p})=-\sum_{k}\mathbf{q}^{\prime(k)}\log\mathbf{p}^{(k)}$ denotes cross-entropy and we simplify $\mathbf{q}(\mathbf{z}_i)$ and $\mathbf{p}(\mathbf{z}_i,\tau_\text{base})$ as $\mathbf{q}_i$ and $\mathbf{p}_i$ respectively. The superscript indicates the $k$-th entry. In Eq.~\eqref{eq:loss-dapl}, two views provide pseudo-labels for each other, like swapped prediction~\cite{caron2020unsupervised}, which implicitly implements consistency regularization~\cite{bachman2014learning}.

\subsubsection{Theoretical Analysis}
We provide the theoretical analysis of the DAPL mechanism. GCD could be viewed as open-world semi-supervised learning (SSL)~\cite{cao2022openworld}, where unlabeled data contains new classes, so it also conforms to the basic assumptions of SSL.
\begin{assumption}[Cluster Assumption~\cite{chapelle2009semi}] \label{assumption:cluster}
    Samples in the same cluster (high-density region) are expected to have the same label.
\end{assumption}

\begin{proposition}[Entropy Minimization~\cite{grandvalet2004semi} in SSL]
\label{proposition:entropy-minimization}
    Under Assumption~\ref{assumption:cluster}, entropy minimization on unlabeled data helps ensure that classes are well-separated.
\end{proposition}
\noindent Entropy minimization~\cite{grandvalet2004semi} could help push unlabeled data to high-density areas away from boundaries, which decreases class overlap and improves inter-class separation.

\begin{proposition}[Pseudo-labeling in SSL]
    Pseudo-labeling~\cite{lee2013pseudo} implicitly performs entropy minimization on unlabeled data.
\end{proposition}
\noindent Generally, learning with pseudo-labels $\hat y$ on unlabeled data $\mathbf{x}_u$ could be expressed as minimizing the cross-entropy between pseudo-labels and the model predictions, \ie, $\mathcal{L}(f(\mathbf{x}_u),\hat y)$. Regardless of whether the pseudo-labels are hard~\cite{lee2013pseudo} or soft~\cite{berthelot2019mixmatch}, they are invariably more confident with lower entropy than the model's predictions, consequently, pseudo-labeling encourages the model to predict more confidently and minimize the entropy on unlabeled data.

Let $R(f)=\mathbb{E}_{(x,y)}\mathcal{L}(f(\mathbf{x}),y)$ denote the true risk of the classification model $f$. The empirical risk could be decomposed as $\widehat R(f)=\widehat R_l(f)+\widehat R_u(f)$, where $\widehat R_l(f)=\frac{1}{n}\sum_{i=1}^n\mathcal{L}(f(\mathbf{x}_i),y_i)$ and $\widehat R_u(f)=\frac{1}{m}\sum_{j=1}^m\mathcal{L}(f(\mathbf{x}_j),\hat y_i)$ are empirical risk on labeled and unlabeled data. The error of hard pseudo-labels $\hat y_j=\argmax_c f(\mathbf{x}_j)[c]$ with threshold $\tau_\text{pl}$ could be written as $\text{err}_\text{pl}=\frac{1}{m}\sum_{j=1}^m \mathds{1}_{[f(\mathbf{x}_j)[\hat y_j]\geq \tau_\text{pl}]}\cdot \mathds{1}_{[\hat y_j\neq y_j]}$. Then we have the theorem~\cite{xie2023classdistributionaware} below:
\begin{theorem}[Performance Gap of Pseudo-labeling Methods~\cite{xie2023classdistributionaware}]
\label{theorem:pseudo-labeling}
    Suppose the loss function $\ell(\cdot)$ is $L_\ell$-Lipschitz continuous and bounded by $B$. For some $\epsilon>0$, if $\text{err}_\text{pl}\leq\epsilon$, and for any $\delta>0$, with probability at least $1-\delta$, we have:
    \begin{equation}
    \small
            R(\hat f)-R(f^\star) \leq 2KB\epsilon+4KL_{\ell}\mathcal{R}_N(\mathcal{F})+2KB\sqrt{\frac{\log(2/\delta)}{2N}},
    \end{equation}
    where $\mathcal{R}_N(\mathcal{F})$ is the expected Rademacher complexity~\cite{mohri2018foundations} and $N=m+n$ denotes the total number of training samples, $f^\star=\argmin_{f\in\mathcal{F}}R(f)$ and $\hat f=\argmin_{f\in\mathcal{F}}\widehat R(f)$ denote the minimizer of true risk $R(f)$ and empirical risk $\widehat R(f)$, respectively. 
\end{theorem}
\noindent From Theorem~\ref{theorem:pseudo-labeling}, the performance of $\hat f$ depends on the error of pseudo-labels and the number of training samples. Lower $\text{err}_\text{pl}$ leads to better generalization performance.

To build a strong classifier, we have to balance between Proposition~\ref{proposition:entropy-minimization} and Theorem~\ref{theorem:pseudo-labeling}. On the one hand, we are supposed to encourage the model to output confident predictions. On the other hand, we should still avoid overconfidence in pseudo-labels, which could bring about severe errors and \textit{confirmation bias}~\cite{pseudoLabel2019}, and it is more obvious in GCD owing to its stricter labeling conditions. In ProtoGCD, we propose DAPL to balance them. Specifically, we progressively provide the model with more confident pseudo-labels as the model's performance improves. To realize this objective, We achieve adaptivity on two levels: (1) We assign hard labels for more confident samples while soft labels for others. (2) The ratio of samples for hard labels increases gradually. In this way, DAPL helps achieve efficient training while circumventing bias.

\begin{assumption}[Consistency Regularization~\cite{bachman2014learning}]   
    The model's predictions remain consistent over some slight perturbations.
\end{assumption}
\noindent ProtoGCD adopts cross-view prediction (Eq.~\eqref{eq:loss-dapl}), which ensures consistency across various augmentations and enhances the model's robustness and generalization ability.

\subsection{Regularization} \label{subsec:regularization}

\subsubsection{Avoiding Trivial Solutions} \label{subsubsec:regularization-trivial}
GCD is essentially a transfer clustering task~\cite{Han_2019_ICCV}, which is susceptible to trivial solutions~\cite{caron2020unsupervised,caron2018deep} where most of the samples in $\mathcal{D}_u$ are allocated to one or a small number of clusters. Early works employ equipartition constraints~\cite{caron2020unsupervised}, which do not always hold for long-tailed data, and others resort to heuristics~\cite{caron2018deep}. In ProtoGCD, we adopt marginal entropy maximization~\cite{hu2017learning} as follows:
\begin{equation}
    \mathcal{L}_\text{entropy}=-H(\overline{\mathbf{p}})=\sum_{k=1}^K \overline{\mathbf{p}}^{(k)}\log \overline{\mathbf{p}}^{(k)},
    \label{eq:loss-entropy}
\end{equation}
where $\mathbf{H}(\mathbf{\mathbf{p}})=-\sum_{k}\mathbf{p}^{(k)}\log \mathbf{p^{(k)}}$ denotes entropy, and $\overline{\mathbf{p}}=\frac{1}{2|\mathcal{B}|}\sum_{i\in\mathcal{B}}\Big(\mathbf{p}(\mathbf{z}_i,\tau_\text{base})+\mathbf{p}(\mathbf{z}^\prime_i,\tau_\text{base})\Big)$ denotes marginal probability distribution over two views. $\mathcal{L}_\text{entropy}$ encourages to predict across different categories as evenly as possible as a whole. We also provide an orthogonal perspective of Eq.~\eqref{eq:loss-entropy} in Theorem~\ref{theorem:entropy}. The proof is in the Appendix.

\begin{theorem}
\label{theorem:entropy}
    Marginal entropy maximization $\mathcal{L}_\text{entropy}$ is equivalent to incorporating a prior distribution $\mathcal{U}$ across $K$ categories, where $\mathcal{U}$ is a uniform distribution.
\end{theorem}

\emph{Advantages of Entropy Regularization.}
The advantages of $\mathcal{L}_\text{entropy}$ are two-fold. Firstly, it is a flexible soft regularization term. One could choose the proper weight and even specific prior distribution according to the characteristics of the downstream datasets, instead of imposing equipartition constraints~\cite{caron2020unsupervised} in all cases. Secondly, $\mathcal{L}_\text{entropy}$ is differentiable and could be learned end-to-end, which is effective without any alternative optimization~\cite{caron2020unsupervised}.

\subsubsection{Inter-Class Separation} \label{subsubsec:regularization-inter-class}
Learning with pseudo-labels (Section~\ref{subsec:pseudo-label}) improves intra-class compactness. It is also important to promote inter-class separation for better classification. To this end, we explicitly increase the distances among prototypes $\mathcal{P}$, \ie, decrease the similarities between each pair of prototypes, and obtain the inter-class separation regularization term as below:
\begin{equation}
    \mathcal{L}_\text{sep}=\frac{1}{K}\sum_{i=1}^K\log \frac{1}{K-1}\sum_{j=1,j\neq i}^K\exp(\boldsymbol{\mu}_i^\top\boldsymbol{\mu}_j/\tau_\text{sep}),
    \label{eq:loss-sep}
\end{equation}
$\tau_\text{sep}$ is the temperature. The overall regularization is:
\begin{equation}
    \mathcal{L}_\text{reg}=\lambda_\text{entropy}\mathcal{L}_\text{entropy}+\lambda_\text{sep}\mathcal{L}_\text{sep},
    \label{eq:loss-regularization}
\end{equation}
where $\lambda_\text{entropy}$ and $\lambda_\text{sep}$ are weights of two terms.

\subsection{Overall Learning Objective}
\label{subsec:overall-objective}

For labeled data $\mathcal{D}_l$, ProtoGCD directly learns from the ground-truth labels on both of the views:
\begin{equation}
    \mathcal{L}_\text{sup}=\frac{1}{2|\mathcal{B}_l|}\sum_{i\in\mathcal{B}_l}\Big(\ell (\mathbf{y}_i^l,\mathbf{p}_i)+\ell (\mathbf{y}_i^l,\mathbf{p}^\prime_i)\Big).
    \label{eq:loss-sup}
\end{equation}
The integrated classification loss, \ie, learning with ground-truth labels on $\mathcal{D}_l$ and pseudo-labels on $\mathcal{D}_l\cup\mathcal{D}_u$, is:
\begin{equation}
    \mathcal{L}_\text{cls}=(1-\lambda_\text{sup})\mathcal{L}_\text{dapl}+\lambda_\text{sup}\mathcal{L}_\text{sup},
    \label{eq:loss-classification}
\end{equation}
where $\lambda_\text{sup}$ and $(1-\lambda_\text{sup})$ denote the weights of supervised and unsupervised components, which is the same as Eq.~\eqref{eq:loss-con}.

By integrating the learning objectives in Section~\ref{subsec:contrastive-learning}---Section~\ref{subsec:regularization}, \ie, $\mathcal{L}_\text{con}$ in Eq.~\eqref{eq:loss-con}, $\mathcal{L}_\text{cls}$ in Eq.~\eqref{eq:loss-classification} and $\mathcal{L}_\text{reg}$ in Eq.~\eqref{eq:loss-regularization}, we could obtain the overall learning objective:
\begin{equation}
    \mathcal{L}=\mathcal{L}_\text{con}+\mathcal{L}_\text{cls}+\mathcal{L}_\text{reg}.
    \label{eq:loss-overall}
\end{equation}

\emph{End-to-end Training.}
Each term in Eq.~\eqref{eq:loss-overall} is differential. The learnable prototypes $\mathcal{P}$, feature extractor $\mathcal{E}(\cdot)$ and projection head $\phi(\cdot)$ could be updated collectively in an end-to-end manner. Consequently, ProtoGCD is an efficient framework without any alternating optimization or EM-like operations like~\cite{Zhao_2023_ICCV,pu2023dynamic}. It also flexibly mitigates confirmation bias and learns appropriate and unbiased representations for GCD.

\section{Estimating the Number of Categories} \label{sec:estimate-number}
In the literature of GCD, most methods assume the number of new categories $K_{new}$ is known \emph{a-prior}, which is unrealistic. It is important to estimate $K_{new}$ given the whole training data $\mathcal{D}=\mathcal{D}_l\cup\mathcal{D}_u$. Vaze \etal~\cite{vaze2022gcd} propose to run K-means~\cite{arthur2007k} on $\mathcal{D}$ with various $K_{new}$, and choose the one corresponding to the maximum clustering accuracy on the labeled data as an estimation of $K_{new}$, namely \textit{Max-Acc}. However, only considering accuracy neglects latent information in feature space and leads to degraded results.
In this paper, we propose to simultaneously exploit the accuracy of labeled data and feature information of all data. Let $\widetilde{K}_{new}$ and $K_{new}$ denote the estimated and ground truth number of new classes. $\widetilde{K}=K_{old}+\widetilde{K}_{new}$.
We train ProtoGCD models with various numbers of classes, \ie, total number of prototypes $\mathcal{P}^{\widetilde{K}}=\{\boldsymbol{\mu}_c\}_{c=1}^{\widetilde{K}}$, and devise the following two proxies.  

\emph{Accuracy Score.}
ProtoGCD adopts the parametric classifier, so we could directly compute old classes' accuracy on $\mathcal{D}_{l}$ without clustering and Hungarian algorithm~\cite{kuhn1955hungarian} as below:
\begin{equation}
    \texttt{accScore} = \frac{1}{|\mathcal{D}_l|}\sum_{i\in\mathcal{D}_l}\mathds{1}_{\big[y_i=\argmax_c p(y=c|\mathbf{z}_i,\tau)\big]},
\end{equation}
$\mathcal{L}_{entropy}$ in Eq.~\eqref{eq:loss-entropy} encourages uniform predictions, if $\widetilde{K}_{new}>K_{new}$, some samples from $\mathcal{C}_{old}$ in $\mathcal{D}_l$ are assigned outside of old prototypes, leading to lower $\texttt{accScore}$.

\emph{Centroid Score.}
The centroids, \ie, mean features, of $\mathcal{C}_{old}$ could be computed in the following two ways:
\begin{align}
    \mathbf{c}_l^k & = \frac{1}{|\mathcal{D}_l^k|}\sum_{i\in\mathcal{D}_l^k} \mathbf{z}_i,\quad & k=1,2,\cdots,K_{old}, \label{eq:centroid-labeled}\\
    \mathbf{c}_u^k & = \frac{1}{|\mathcal{D}_u^k|}\sum_{i\in\mathcal{D}_u^k} \mathbf{z}_i,\quad & k=1,2,\cdots,K_{old}, \label{eq:centroid-unlabeled}
\end{align}
where $\mathcal{D}_l^k=\{(\mathbf{x}_i,y_i)\in\mathcal{D}_l,y_i=k\}$ denotes the labeled samples assigned to the prototypes of old classes based on ground-truth labels $y_i$, and $\mathcal{D}_u^k=\{(\mathbf{x}_i)\in\mathcal{D}_u,\tilde{y}_i=k\}$ denotes the unlabeled samples assigned to the prototypes of old classes based on the model's predictions $\widetilde{y}_i=\argmax_c p(y=c|\mathbf{z}_i,\tau)$.
Similarly, due to the effect of $\mathcal{L}_{entropy}$, if $\widetilde{K}_{new}<K_{new}$, more samples from $\mathcal{C}_{new}$ in $\mathcal{D}_u$ are assigned to old classes, in this case, the divergence between $\mathbf{c}_l^k$ and the corresponding $\mathbf{c}_u^k$ in old classes becomes larger, resulting in lower $\texttt{centrScore}$:
\begin{equation}
    \texttt{centrScore} = \prod_{k=1}^{K_{old}}\mathbf{c}_l^{k\top} \mathbf{c}_u^k,
\end{equation}

\begin{algorithm}[!t]
    \small
    \caption{\textit{Prototype Score} for Class Number Estimation}
    \label{alg:prototype-score}
    \begin{algorithmic}[1]
        \Require {Training dataset $\mathcal{D}=\mathcal{D}_l\cup\mathcal{D}_u$.}
        \Require {Number of old classes $K_{old}$.}
        \Require {Maximum range of new classes number $K_{new}^\mathsf{max}$.}
        \State $\triangleright$ {Initialize the left and right boundary $K_a=0, K_b=K_{new}^\mathsf{max}$.}
        \While {$K_a<K_b$}
            \State $\triangleright$ $K_{c_1}\leftarrow\lfloor \frac{1}{2}(K_a+K_b)\rfloor,\quad K_{c_2}\leftarrow\lfloor \frac{1}{2}(K_a+K_b)\rfloor+1$.
            \State $\triangleright$ {Train ProtoGCD with $(K_{old}+K_{c_1})$ and $(K_{old}+K_{c_2})$ prototypes on $\mathcal{D}$ for $3$ epochs and compute \texttt{protoScore} $p_{c_1}$ and $p_{c_2}$, respectively, as described in Eq.~\eqref{eq:prototype-score}.}
            \If {$p_{c_1}<p_{c_2}$}
                \State $\triangleright$ $K_a\leftarrow K_{c_2},\quad p_a\leftarrow p_{c_2}$.
            \Else
                \State $\triangleright$ $K_b\leftarrow K_{c_1},\quad p_b\leftarrow p_{c_1}$.
            \EndIf
        \EndWhile
        \Ensure {Estimated number of new class $\widetilde{K}^\star_{new}=K_a$.}
    \end{algorithmic}
\end{algorithm}

\emph{Prototype Score as Combination of Two Scores.}
As mentioned above, when $\widetilde{K}_{new}>K_{new}$, $\texttt{accScore}$ becomes lower, when $\widetilde{K}_{new}<K_{new}$, $\texttt{centrScore}$ becomes lower, which motivates us to integrate them and propose \textit{Prototype Score} by incorporating both accuracy and centroids' divergence:
\begin{equation}
    \texttt{protoScore}(\widetilde{K}_{new})=\texttt{accScore}\times \texttt{centrScore}.
    \label{eq:prototype-score}
\end{equation}
In both cases, $\texttt{protoScore}$ is small. We choose the $\widetilde{K}_{new}$ that maximizes \texttt{protoScore} as an estimator of $K_{new}$:
\begin{equation}
    \widetilde{K}^\star_{new}=\argmax_{\widetilde{K}_{new}}\ \texttt{protoScore}(\widetilde{K}_{new}).
    \label{eq:argmax-prototype-score}
\end{equation}

\emph{Overall Pipeline.}
We train ProtoGCD with various class numbers $\widetilde{K}$ and set the prototypes $\mathcal{P}^{\widetilde{K}}=\{\boldsymbol{\mu}_c\}_{c=1}^{\widetilde{K}}$. Models are trained using the overall objectives in Section~\ref{subsec:overall-objective} for only 3 epochs, then we compute $\texttt{protoScore}$ and estimate the novel classes number as in Eq.~\eqref{eq:argmax-prototype-score}. This \textit{low-epoch-training} avoids low distinguishable $\texttt{accScore}$ due to the overfitting to $\mathcal{D}_l$ and ensures fast estimation. We employ a binary search to iterate over $\widetilde{K}$ to further accelerate the algorithm. The whole process is shown in Algorithm~\ref{alg:prototype-score}. Our algorithm requires approximately $O(\log K_{new}^\mathsf{max})$ epochs. After the acquisition of $\widetilde{K}^\star_{new}$, we could use the estimated number to instantiate prototypes and train models for GCD with the proposed method in Section~\ref{sec:method-protogcd}.

\section{Extending to Detect Unseen Outliers}
\label{sec:extend-to-ood}
Once trained on partially labeled old classes $\mathcal{C}_{old}$ and unlabeled new classes $\mathcal{C}_{new}$, the model can classify samples from $\mathcal{C}_{old}$ and $\mathcal{C}_{new}$ during testing. However, in practical scenarios, test samples outside of $\mathcal{C}_{old}\cup \mathcal{C}_{new}$ could emerge after the model's deployment, we refer to them as \emph{outliers} or \emph{unseen novel categories}, and denote them as $\mathcal{C}_{out}$, see Fig.~\ref{fig:unification} (c). Since the model has not seen $\mathcal{C}_{out}$ during training, it is essential to detect them during inference, rather than irresponsibly classifying them into one of the categories in $\mathcal{C}_{old}\cup \mathcal{C}_{new}$, which is important in safety-critical circumstances~\cite{9863660} and often overlooked in GCD~\cite{vaze2022gcd}.

In this paper, we extend ProtoGCD to not only classify $\mathcal{C}_{old}$ and cluster $\mathcal{C}_{new}$, but also to reject $\mathcal{C}_{out}$. Herein, we refer to $\mathcal{C}_{old}\cup \mathcal{C}_{new}$ as in-distribution (ID) and $\mathcal{C}_{out}$ as out-of-distribution (OOD). In other words, we extend ProtoGCD to the task of OOD detection~\cite{hendrycks2016baseline}. Following the common practice, we assign each sample $\mathbf{x}$ a confidence score $S(\mathbf{x})$, indicating its \emph{normality}. Given a pre-defined threshold $\delta_\text{ood}$, if $S(\mathbf{x})\geq \delta_\text{ood}$, then $\mathbf{x}$ is recognized as ID, otherwise, $\mathbf{x}$ is detected as OOD and rejected. Because ProtoGCD adopts the parametric classifier, it could easily obtain the predictive probability, as in Eq.~\eqref{eq:posterior-prob}. We propose to employ the post-hoc score functions for OOD detection, like MSP~\cite{hendrycks2016baseline} and Energy~\cite{liu2020energy}, these methods are independent of ProtoGCD's training, thus could be directly integrated into our method for OOD detection, \eg, $S(\mathbf{x})=\max_k p(y=k|\mathbf{z},\tau)$ for MSP. By contrast, methods~\cite{vaze2022gcd,fei2022xcon,pu2023dynamic} using contrastive learning could not directly obtain posterior probabilities. We propose to firstly run K-means~\cite{arthur2007k} on training data and obtain the cluster centroids for ID classes, which are then used to compute probabilities similar to Eq.~\eqref{eq:posterior-prob}.

\begin{table}[!tb]
\setlength\tabcolsep{6pt}
\centering
\renewcommand{\arraystretch}{1}
\caption{The statistics of three \colorbox{color2}{generic datasets} and three \colorbox{color3}{fine-grained datasets}. The number of instances of both labeled and unlabeled data is shown ($|\mathcal{D}_l|$, $|\mathcal{D}_u|$), as well as the number of classes ($|\mathcal{Y}_l|=K_{old}$, $|\mathcal{Y}_u|=K_{old}+K_{new}$).}
\vspace{-5pt}
\label{tab:datasets}
\resizebox{.9\linewidth}{!}{
\begin{tabular}{@{}ccccc@{}}
\toprule
\multirow{2}{*}{Datasets} & \multicolumn{2}{c}{Labeled $\mathcal{D}_l$} & \multicolumn{2}{c}{Unlabeled $\mathcal{D}_u$} \\ \cmidrule(l){2-3} \cmidrule(l){4-5} 
 & $|\mathcal{D}_l|$ & $|\mathcal{Y}_l|$ & $|\mathcal{D}_u|$ & $|\mathcal{Y}_u|$ \\ \midrule
\cellcolor{color2}CIFAR10~\cite{krizhevsky2009learning} & 12,500 & 5 & 37,500 & 10 \\
\cellcolor{color2}CIFAR100~\cite{krizhevsky2009learning} & 20,000 & 80 & 30,000 & 100 \\
\cellcolor{color2}ImageNet-100~\cite{deng2009imagenet} & 31,860 & 50 & 95,255 & 100 \\ \midrule \midrule
\cellcolor{color3}CUB~\cite{wah2011caltech} & 1,498 & 100 & 4,496 & 200 \\
\cellcolor{color3}Stanford Cars (SCars)~\cite{krause20133d} & 2,000 & 98 & 6,144 & 196 \\
\cellcolor{color3}FGVC-Aircraft (Aircraft)~\cite{maji2013fine} & 1,666 & 50 & 5,001 & 100 \\
\cellcolor{color3}Herbarium19 (Herb)~\cite{tan2019herbarium} & 8,869 & 341 & 25,356 & 683 \\ \bottomrule
\end{tabular}
}
\end{table}

\begin{table*}[!t]
    \setlength{\tabcolsep}{3pt}
    \centering
    \renewcommand{\arraystretch}{1.0}
    \caption{Main results on generic image classification datasets, where $\dagger$ denotes the reproduced results.}
    \vspace{-5pt}
    \label{tab:main-generic}
    \resizebox{.9\textwidth}{!}{
    \begin{tabular}{@{}llllllllll@{}}
    \toprule
    \multirow{2}{*}{Methods} & \multicolumn{3}{c}{CIFAR10} & \multicolumn{3}{c}{CIFAR100} & \multicolumn{3}{c}{ImageNet-100} \\ \cmidrule(l){2-4} \cmidrule(l){5-7} \cmidrule(l){8-10} 
     & All & Old & New & All & Old & New & All & Old & New \\ \midrule
    K-means~\cite{arthur2007k} & 83.6 & 85.7 & 82.5 & 52.0 & 52.2 & 50.8 & 72.7 & 75.5 & 71.3 \\
    RankStats+~\cite{9464163} & 46.8 & 19.2 & 60.5 & 58.2 & 77.6 & 19.3 & 37.1 & 61.6 & 24.8 \\
    UNO+~\cite{Fini_2021_ICCV} & 68.6 & \textbf{98.3} & 53.8 & 69.5 & 80.6 & 47.2 & 70.3 & \textbf{95.0} & 57.9 \\
    ORCA$^\dagger$~\cite{cao2022openworld} & 81.8 & 86.2 & 79.6 & 69.0 & 77.4 & 52.0 & 73.5 & 92.6 & 63.9 \\
    GCD~\cite{vaze2022gcd} & 91.5 & 97.9 & 88.2 & 73.0 & 76.2 & 66.5 & 74.1 & 89.8 & 66.3 \\
    XCon~\cite{fei2022xcon} & 96.0 & 97.3 & 95.4 & 74.2 & 81.2 & 60.3 & 77.6 & 93.5 & 69.7 \\
    DCCL~\cite{pu2023dynamic} & 96.3 & 96.5 & 96.9 & 75.3 & 76.8 & 70.2 & 80.5 & 90.5 & 76.2 \\ 
    {GPC}~\cite{Zhao_2023_ICCV} & {92.2} & {98.2} & {89.1} & {77.9} & {\textbf{85.0}} & {63.0} & {76.9} & {94.3} & {71.0} \\ 
    {SimGCD}~\cite{wen2023parametric} & {97.1} & {95.1} & {98.1} & {80.1} & {81.2} & {77.8} & {83.0} & {93.1} & {77.9} \\ \midrule
    \RC{30}ProtoGCD (ours) & \textbf{97.3$_{\pm 0.0}$} & 95.3$_{\pm 0.2}$ & \textbf{98.2$_{\pm 0.1}$} & \textbf{81.9$_{\pm 0.2}$} & 82.9$_{\pm 0.0}$ & \textbf{80.0$_{\pm 0.4}$} & \textbf{84.0$_{\pm 0.6}$} & 92.2$_{\pm 0.9}$ & \textbf{79.9$_{\pm 1.3}$} \\ \bottomrule
    \end{tabular}
    }
\end{table*}

\section{Experiments} \label{sec:experiments}

\subsection{Experimental Setup}

\emph{Datasets.}
we conduct experiments on generic recognition datasets: CIFAR10~\cite{krizhevsky2009learning}, CIFAR100~\cite{krizhevsky2009learning} and ImageNet-100~\cite{deng2009imagenet}, as well as more challenging fine-grained datasets in Semantic Shift Benchmark~\cite{vaze2022openset}: CUB~\cite{wah2011caltech}, Stanford Cars (SCars)~\cite{krause20133d}, FGVC-Aircraft (Aircraft)~\cite{maji2013fine} and Herbarium19 (Herb)~\cite{tan2019herbarium}. Following the canonical setting in the literature of GCD~\cite{vaze2022gcd,fei2022xcon,pu2023dynamic}, in each dataset, we sample a subset of all classes as old classes $\mathcal{C}_{old}$, the remaining classes are novel classes $\mathcal{C}_{new}$. Half of the instances in old classes from the original training data are drawn to form labeled data $\mathcal{D}_l$, while all the remaining data from the original training set constitute the unlabeled dataset $\mathcal{D}_u$. We summarize the datasets' statistics in Table~\ref{tab:datasets}. The original test data in each dataset serves as the validation set for model selection. GCD follows the transductive setting~\cite{vaze2022gcd}, \ie, the model is trained on $\mathcal{D}_l\cup\mathcal{D}_u$ and evaluated on $\mathcal{D}_u$.

\begin{table*}[!t]
    \setlength{\tabcolsep}{2pt}
    \centering
    \renewcommand{\arraystretch}{1.0}
    \caption{Main results on fine-grained image classification datasets, where $\dagger$ denotes the reproduced results.}
    \vspace{-5pt}
    \label{tab:main-fine-grained}
    \resizebox{.9\textwidth}{!}{
    \begin{tabular}{@{}lllllllllllll@{}}
    \toprule
    \multirow{2}{*}{Methods} & \multicolumn{3}{c}{CUB} & \multicolumn{3}{c}{Stanford Cars} & \multicolumn{3}{c}{FGVC-Aircraft} & \multicolumn{3}{c}{Herbarium19} \\ \cmidrule(l){2-4} \cmidrule(l){5-7} \cmidrule(l){8-10} \cmidrule(l){11-13} 
     & All & Old & New & All & Old & New & All & Old & New & All & Old & New \\ \midrule
    K-means~\cite{arthur2007k} & 34.3 & 38.9 & 32.1 & 12.8 & 10.6 & 13.8 & 16.0 & 14.4 & 16.8 & 13.0 & 12.2 & 13.4 \\
    RankStats+~\cite{9464163} & 33.3 & 51.6 & 24.2 & 28.3 & 61.8 & 12.1 & 26.9 & 36.4 & 22.2 & 27.4 & 55.8 & 12.8 \\
    UNO+~\cite{Fini_2021_ICCV} & 35.1 & 49.0 & 28.1 & 35.5 & 70.5 & 18.6 & 40.3 & 56.4 & 32.2 & 28.3 & 53.7 & 14.7 \\
    ORCA$^\dagger$~\cite{cao2022openworld} & 35.3 & 45.6 & 30.2 & 31.9 & 42.2 & 26.9 & 31.6 & 32.0 & 31.4 & 24.6 & 26.5 & 23.7 \\
    GCD~\cite{vaze2022gcd} & 51.3 & 56.6 & 48.7 & 39.0 & 57.6 & 29.9 & 45.0 & 41.1 & 46.9 & 35.4 & 51.0 & 27.0 \\
    XCon~\cite{fei2022xcon} & 52.1 & 54.3 & 51.0 & 40.5 & 58.8 & 31.7 & 47.7 & 44.4 & 49.4 & 38.1$^\dagger$ & 58.3$^\dagger$ & 27.3$^\dagger$ \\
    DCCL~\cite{pu2023dynamic} & \textbf{63.5} & 60.8 & \textbf{64.9} & 43.1 & 55.7 & 36.2 & ~~--~~ & ~~--~~ & ~~--~~ & ~~--~~ & ~~--~~ & ~~--~~ \\ 
    {GPC}~\cite{Zhao_2023_ICCV} & {55.4} & {58.2} & {53.1} & {42.8} & {59.2} & {32.8} & {46.3} & {42.5} & {47.9} & ~~--~~ & ~~--~~ & ~~--~~  \\ 
    {SimGCD}~\cite{wen2023parametric} & {60.3} & {65.6} & {57.7} & {\textbf{53.8}} & {71.9} & {\textbf{45.0}} & {54.2} & {59.1} & {51.8} & {44.0} & {58.0} & {36.4} \\ \midrule
    \RC{30}ProtoGCD (ours) & 63.2$_{\pm 0.1}$ & \textbf{68.5$_{\pm 0.5}$} & 60.5$_{\pm 0.2}$ & \textbf{53.8$_{\pm 0.4}$} & \textbf{73.7$_{\pm 0.6}$} & 44.2$_{\pm 0.6}$ & \textbf{56.8$_{\pm 0.4}$} & \textbf{62.5$_{\pm 0.8}$} & \textbf{53.9$_{\pm 0.9}$} & \textbf{44.5$_{\pm 0.3}$} & \textbf{59.4$_{\pm 0.5}$} & \textbf{36.5$_{\pm 0.4}$} \\
    \bottomrule
    \end{tabular}
    }
\end{table*}

\emph{Evaluation Protocol.}
GCD is essentially a clustering problem, we evaluate the performance following~\cite{vaze2022gcd}. At test time, we measure the clustering accuracy (ACC) of the model's predictions $\tilde y_i$ given the ground-truth labels $y_i$:
\begin{equation}
    ACC=\max_{\omega\in\Omega(\mathcal{Y}_u)}\frac{1}{M}\sum_{i=1}^M\mathds{1}\big\{y_i=\omega(\tilde y_i)\big\},
    \label{eq:acc}
\end{equation}
where $M=|\mathcal{D}_u|$ denotes the total number of unlabeled samples, and $\Omega(\mathcal{Y}_u)$ represents the set of all permutations that match the prediction to the ground-truth labels. We find the optimal permutation by the Hungarian algorithm~\cite{kuhn1955hungarian}, which is performed only \emph{once} across both $\mathcal{C}_{old}$ and $\mathcal{C}_{new}$ on all the unlabeled data~\cite{vaze2022gcd}. The ACC in Eq.~\eqref{eq:acc} reflects the overall clustering performance on the entire unlabeled dataset $\mathcal{D}_u$, namely `All', we further report the clustering accuracy for samples from the old classes $\mathcal{C}_{old}$ subset and the new classes $\mathcal{C}_{new}$ subset in $\mathcal{D}_u$, namely `Old' and `New' respectively. The `Old' and `New' results are evaluated after the Hungarian assignment is computed.

\emph{Implementation Details.}
For fair comparisons, we follow prior arts~\cite{vaze2022gcd,fei2022xcon,pu2023dynamic} and train our method with ViT-B/16 backbone~\cite{dosovitskiy2021an} pre-trained with DINO~\cite{caron2021emerging}, and the final transformer block is fine-tuned. We use the output $\texttt{[CLS]}$ token as feature representation $\mathbf{z}_i$. All the methods are trained for 200 epochs with a batch size of 128, and models are selected on the validation set for evaluation. The feature and projection space dimensions are 768 and 65,536, as in~\cite{vaze2022gcd}. The initial learning rate is 0.1 and decayed with a cosine annealed schedule. As for the hyper-parameters, the weight of the supervised component $\lambda_\text{sup}$ is 0.35. $\lambda_\text{entropy}$ and $\lambda_\text{sep}$ is set to be 2 and 0.1 respectively. $\tau_\text{base}=\tau_\text{sep}=0.1$, and $\tau_\text{sharp}=0.05$. The ramp-up stage contains $e_\text{ramp}=100$ epochs with a linear schedule as in Eq.~\eqref{eq:ratio-ramp-up}. All experiments are conducted on NVIDIA RTX A6000 GPUs.

\subsection{Generalized Category Discovery Performance}

\subsubsection{Comparison with State-of-the-Arts}

We compare our method with naive K-means~\cite{arthur2007k}, strong baselines~\cite{9464163,Fini_2021_ICCV} derived from NCD and competitive GCD methods~\cite{vaze2022gcd,fei2022xcon,cao2022openworld} DCCL~\cite{pu2023dynamic}, {GPC~\cite{Zhao_2023_ICCV} and state-of-the-art (SOTA) SimGCD~\cite{wen2023parametric} and $\mu$GCD~\cite{NEURIPS2023_3f52ab43}}. We report the results of our method averaged over 5 runs (mean $\pm$ std), while for other methods, official results from original papers are reported. The experimental results on generic and fine-grained image datasets are shown in Table~\ref{tab:main-generic} and Table~\ref{tab:main-fine-grained}, respectively.

\emph{ProtoGCD outperforms previous SOTA methods by a large margin.}
ProtoGCD consistency achieves remarkable performance. For example, on CIFAR100, ProtoGCD achieves $1.8\%$ gains on `All' classes and $2.2\%$ on `New' classes, as in Table~\ref{tab:main-generic}. For fine-grained datasets in Table~\ref{tab:main-fine-grained}, our method outperforms DCCL~\cite{pu2023dynamic} by $7.0\%$ on SCars. The results indicate that ProtoGCD learns better representations from the pseudo-labeling mechanism and parametric prototypes.

\emph{ProtoGCD provides more balanced accuracy between old and novel classes.}
The significant issue addressed by ProtoGCD is the imbalanced performance between old and new classes, especially for parametric classifier-based methods~\cite{9464163,Fini_2021_ICCV}. On ImageNet-100, although UNO+~\cite{Fini_2021_ICCV} achieves the best `Old' accuracy, it suffers from severely imbalanced performance ($37.1\%$ gap between `Old' and `New'). By contrast, our method achieves more balanced results ($12.3\%$). A similar trend could be observed in other datasets. These results show that ProtoGCD benefits from its unified modeling and learning objectives between old and new classes to obtain balanced accuracy.

\subsubsection{Inductive Evaluation} \label{subsubsec:inductive}
Canonical GCD follows transductive evaluation~\cite{vaze2022gcd,fei2022xcon,pu2023dynamic}, \ie, models are tested on the unlabeled part $\mathcal{D}_u$ of training data. In this paper, we generalize to the inductive evaluation, where we evaluate the trained models on separate and unseen test datasets. The results are shown in Table~\ref{tab:inductive-setting}. Compared with the transductive results, contrastive learning-based methods GCD~\cite{vaze2022gcd} and XCon~\cite{fei2022xcon} have degraded performance. The reason is that these methods use semi-supervised K-means for transductive evaluation, however, there are no labeled data at hand for inductive settings, and unsupervised K-means results in unstable clusters. Our method utilizes a parametric classifier and does not rely on $\mathcal{D}_l$ at inference time. Consequently, ProtoGCD achieves better generalization performance under inductive settings as in Table~\ref{tab:inductive-setting}. For instance, the performance degradation of our method on Aircraft is $0.2\%$, less than $6.2\%$ of XCon~\cite{fei2022xcon}.

\begin{table}[!t]
    \setlength\tabcolsep{4pt}
    \centering
    \renewcommand{\arraystretch}{1.0}
    \caption{Inductive evaluation on four datasets. Values in $()$ indicate the performance gap compared with transductive evaluations, \ie, generalization errors.}
    \vspace{-5pt}
    \label{tab:inductive-setting}
    \resizebox{.85\linewidth}{!}{
    \begin{tabular}{@{}ccccc@{}}
    \toprule
    \multicolumn{2}{c}{Datasets} & GCD & XCon & Ours \\ \midrule
    \multirow{2}{*}{CIFAR100} & Old & 75.4 $\color{Red}(0.8\downarrow)$ & 81.1 $\color{Red}(0.1\downarrow)$ & \textbf{82.5 $\color{Red}(0.4\downarrow)$} \\
     & New & 60.0 $\color{Red}(6.5\downarrow)$ & 51.5 $\color{Red}(8.8\downarrow)$ & \textbf{78.0 $\color{Red}(2.0\downarrow)$} \\ \midrule
    \multirow{2}{*}{ImageNet-100} & Old & 87.3 $\color{Red}(2.5\downarrow)$ & 91.4 $\color{Red}(2.1\downarrow)$ & \textbf{92.8 $\color{Green}(0.3\uparrow)$} \\
     & New & 65.4 $\color{Red}(0.9\downarrow)$ & 64.4 $\color{Red}(5.3\downarrow)$ & \textbf{78.4 $\color{Red}(1.5\downarrow)$} \\ \midrule
    \multirow{2}{*}{SCars} & Old & 52.0 $\color{Red}(5.6\downarrow)$ & 53.8 $\color{Red}(5.0\downarrow)$ & \textbf{68.9 $\color{Red}(0.3\downarrow)$} \\
     & New & 26.6 $\color{Red}(3.3\downarrow)$ & 27.4 $\color{Red}(4.3\downarrow)$ & \textbf{41.2 $(0.0\downarrow)$} \\ \midrule
    \multirow{2}{*}{Aircraft} & Old & 40.1 $\color{Red}(1.1\downarrow)$ & 43.8 $\color{Red}(0.6\downarrow)$ & \textbf{62.1 $\color{Red}(0.4\downarrow)$} \\
     & New & 41.1 $\color{Red}(5.8\downarrow)$ & 43.2 $\color{Red}(6.2\downarrow)$ & \textbf{53.7 $\color{Red}(0.2\downarrow)$} \\ \bottomrule
    \end{tabular}
    }
\end{table}

\begin{table}[!tb]
    \setlength{\tabcolsep}{3pt}
    \centering
    \renewcommand{\arraystretch}{1.0}
    \caption{{Comparison results using \colorbox{color2}{DINO} and \colorbox{color3}{DINOv2} initialized backbone. \textbf{Bold} and \underline{underline} denote the best and the second best values.}}
    \vspace{-5pt}
    \label{tab:comparison-dinov2-upgrade}
    \resizebox{.98\linewidth}{!}{
    \begin{tabular}{@{}lccccccccc@{}}
    \toprule
    \multirow{2}{*}{Method} & \multicolumn{3}{c}{CUB} & \multicolumn{3}{c}{Stanford Cars} & \multicolumn{3}{c}{FGVC Aircraft} \\ \cmidrule(l){2-10} 
     & All & Old & New & All & Old & New & All & Old & New \\ \midrule
    \RD{30}\multicolumn{10}{c}{DINO} \\ \midrule
    SimGCD~\cite{wen2023parametric} & 60.3 & 65.6 & 57.7 & 53.8 & 71.9 & 45.0 & 54.2 & 59.1 & 51.8 \\
    $\mu$GCD~\cite{NEURIPS2023_3f52ab43} & \underline{65.7} & 68.0 & \underline{64.6} & \underline{56.5} & 68.1 & \underline{50.9} & 53.8 & 55.4 & 53.0 \\ \midrule
    ProtoGCD (ours) & 63.2 & \underline{68.5} & 60.5 & 53.8 & \underline{73.7} & 44.2 & \underline{56.8} & \textbf{62.5} & \underline{53.9} \\
    ProtoGCD+ (ours) & \textbf{66.3} & \textbf{68.9} & \textbf{65.0} & \textbf{58.8} & \textbf{75.1} & \textbf{51.2} & \textbf{59.5} & \underline{62.0} & \textbf{58.3} \\ \midrule \midrule
    \RE{30}\multicolumn{10}{c}{DINOv2} \\ \midrule
    SimGCD~\cite{wen2023parametric} & 71.5 & 78.1 & 68.3 & 71.5 & 81.9 & 66.6 & 63.9 & 69.9 & 60.9 \\
    $\mu$GCD~\cite{NEURIPS2023_3f52ab43} & 74.0 & 75.9 & \textbf{73.1} & \underline{76.1} & \textbf{91.1} & 68.9 & 66.3 & 68.7 & 65.1 \\ \midrule
    ProtoGCD (ours) & \underline{74.9} & \underline{80.1} & 72.3 & 75.8 & 88.7 & \underline{69.5} & \underline{69.4} & \underline{75.9} & \underline{66.2} \\
    ProtoGCD+ (ours) & \textbf{75.7} & \textbf{81.5} & \underline{72.9} & \textbf{77.6} & \underline{90.5} & \textbf{71.5} & \textbf{71.1} & \textbf{76.3} & \textbf{68.5} \\ \bottomrule
    \end{tabular}
    }
\end{table}

\subsubsection{Evaluation under Other Training Configurations}

To comprehensively evaluate our method, we conduct experiments under different training configurations. From the model perspective, we consider a more recent DINOv2~\cite{oquab2023dinov2} for enhanced initializations. From the training techniques perspective, a recent work $\mu$GCD~\cite{NEURIPS2023_3f52ab43} builds upon SimGCD and further utilizes FixMatch~\cite{sohn2020fixmatch}-like techniques, including the exponential moving average of the teacher model and misaligned data augmentations for teacher and student models. $\mu$GCD also employs the model trained in~\cite{vaze2022gcd} for initialization. These techniques are complementary to ProtoGCD. Thus, we seamlessly incorporate the three techniques into ProtoGCD and name the upgraded method as \textbf{ProtoGCD+}. Results under these training configurations are shown in Table~\ref{tab:comparison-dinov2-upgrade}. Our method outperforms SimGCD for both DINO and DINOv2, and the upgraded version ProtoGCD+ achieves the SOTA performance.

\subsubsection{Finding the Number of Classes}
For GCD~\cite{vaze2022gcd,fei2022xcon}, most methods assume the number of new classes is known. To relax this restriction, we present \textit{Prototype Score} for class number estimation in Algorithm~\ref{alg:prototype-score}. We compare our method with GCD~\cite{vaze2022gcd} {Xcon~\cite{fei2022xcon} and DCCL~\cite{pu2023dynamic}} in Table~\ref{tab:estimate-k}. \emph{Prototype Score} consistently achieves more precise estimation results. The reason is that we further consider information in the feature space beyond accuracy and grasp more latent characteristics.

\begin{table}[!t]
    \setlength{\tabcolsep}{4pt}
    \centering
    \renewcommand{\arraystretch}{0.8}
    \caption{{Estimating the number of total classes $K$ in the unlabeled data $\mathcal{D}_u$ on \colorbox{color2}{generic} and \colorbox{color3}{fine-grained} datasets. Here, `GT' denotes the ground truth.}}
    \vspace{-5pt}
    \label{tab:estimate-k}
    \resizebox{.9\linewidth}{!}{
    \begin{tabular}{@{}cccccc@{}}
    \toprule
    Datasets & GT & GCD~\cite{vaze2022gcd} & XCon~\cite{fei2022xcon} & DCCL~\cite{pu2023dynamic} & Ours \\ \midrule
    \cellcolor{color2}CIFAR10 & 10 & 9 & 8 & 14 & \textbf{10} \\
    \cellcolor{color2}CIFAR100 & 100 & \textbf{100} & 97 & 146 & \textbf{100} \\
    \cellcolor{color2}IN-100 & 100 & 109 & 109 & 129 & \textbf{106} \\ \midrule \midrule
    \cellcolor{color3}CUB & 200 & 231 & 236 & 172 & \textbf{211} \\
    \cellcolor{color3}SCars & 196 & 230 & 206 & \textbf{192} & 205 \\
    \cellcolor{color3}Herb & 683 & 520 & - & - & \textbf{603} \\ \bottomrule
    \end{tabular}
    }
\end{table}

\begin{figure}[!t]
    \centering
    \begin{subfigure}{.48\linewidth}
        \centering
        \includegraphics[width=\linewidth]{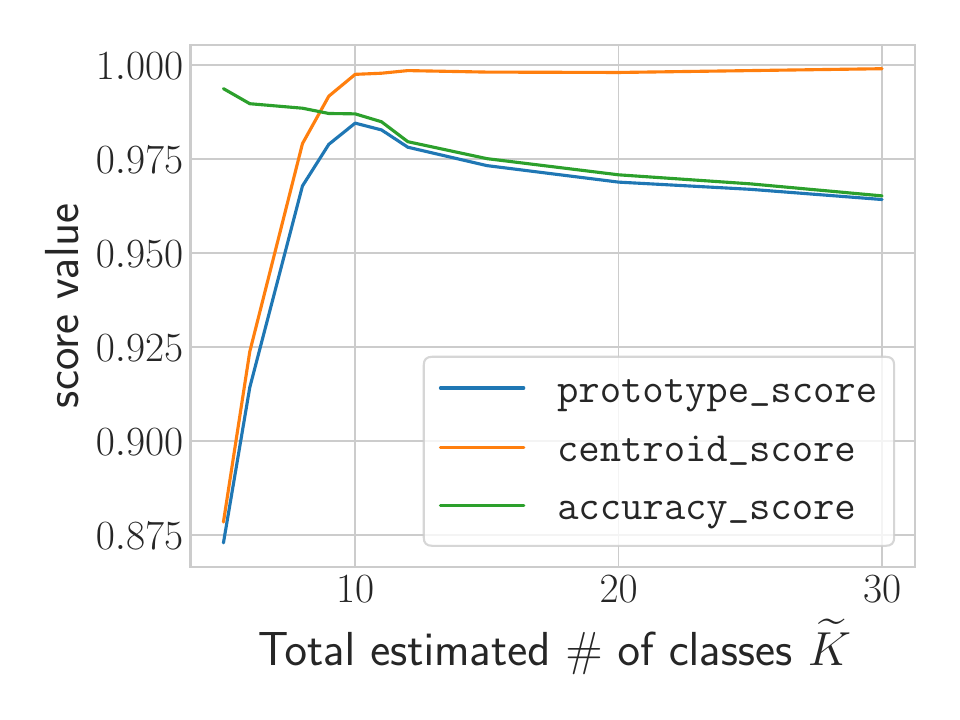}
        \vspace{-15pt}
        \caption{CIFAR10.}
        \label{subfig:estimate-score-cifar10}
    \end{subfigure}
    \hfill
    \begin{subfigure}{.48\linewidth}
        \centering
        \includegraphics[width=\linewidth]{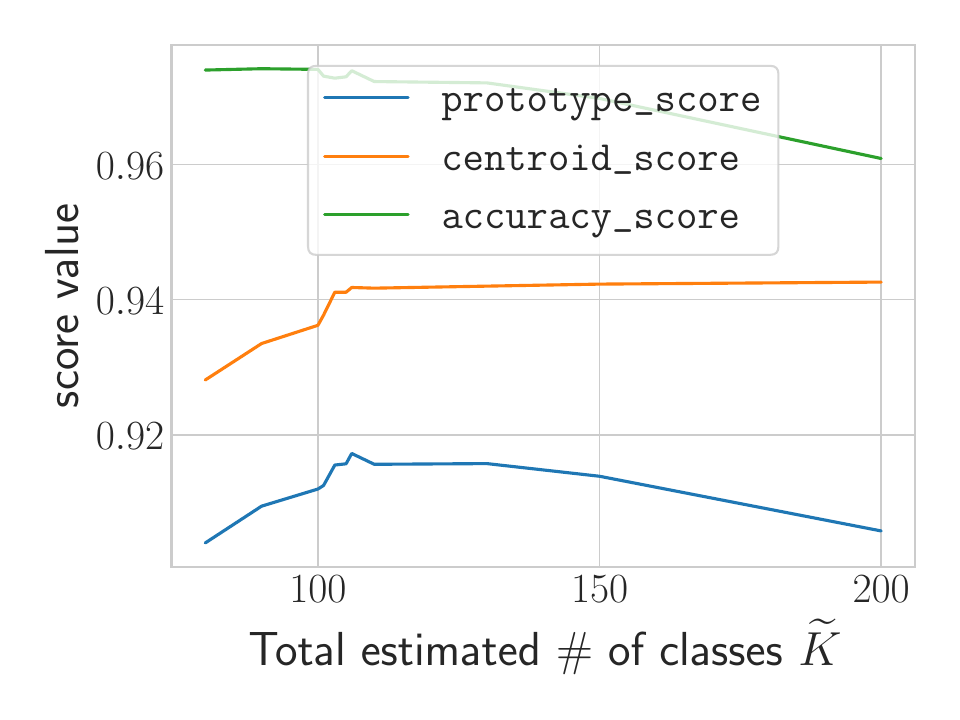}
        \vspace{-15pt}
        \caption{ImageNet-100.}
        \label{subfig:estimate-score-imagenet}
    \end{subfigure}
    \vspace{-5pt}
    \caption{Results on different scores for class number estimation on CIFAR10 (a) and ImageNet-100 (b), and the ground-truth classes numbers $\widetilde{K}$ are 10 and 100, respectively.}
    \label{fig:estimate-score}
\end{figure}

\begin{table}[!t]
    \setlength{\tabcolsep}{6pt}
    \centering
    \renewcommand{\arraystretch}{1.0}
    \caption{Class number estimation results across different training epochs. Here, `GT' denotes the ground truth.}
    \vspace{-5pt}
    \label{tab:estimate-num-epochs}
    \resizebox{.85\linewidth}{!}{
    \begin{tabular}{@{}ccccccccc@{}}
    \toprule
    \multirow{2}{*}{Datasets} & \multirow{2}{*}{GT} & \multicolumn{7}{c}{\# Training Epochs} \\ \cmidrule(l){3-9} 
     &  & 1 & 2 & 3 & 4 & 5 & 6 & 7 \\ \midrule
    C100 & 100 & 89 & 98 & \textbf{100} & 101 & 109 & 120 & 117 \\
    IN-100 & 100 & 90 & \textbf{97} & 106 & 109 & 113 & 119 & 119 \\
    CUB & 200 & 180 & \textbf{205} & 211 & 218 & 221 & 229 & 235 \\
    Herb & 683 & 571 & 595 & 603 & 636 & \textbf{670} & 701 & 724 \\ \bottomrule
    \end{tabular}
    }
\end{table}

To demonstrate the validity of our method, we illustrate the trend of changes in two scores of \emph{Prototype Score}. Fig.~\ref{fig:estimate-score} demonstrates that as the estimated number $\widetilde K$ grows, \texttt{centrScore} increases while \texttt{accScore} decreases. This is consistent with the analysis in Section~\ref{sec:estimate-number}. As a result, we select $\widetilde K$ as the estimation when the combination value of \texttt{centrScore} and \texttt{accScore} is the largest.

\emph{Training epochs for class number estimation.} Algorithm~\ref{alg:prototype-score} requires repeatedly training the model for several epochs. We conduct experiments of class number estimation with different epochs in Table~\ref{tab:estimate-num-epochs}. If the number of training epochs is insufficient, the model is very weak on labeled classes, resulting in unreliable \texttt{accScore}. Conversely, if the number of epochs is large, the model tends to overfit the labeled data, resulting in indistinguishable \texttt{accScore}. Then, \texttt{centrScore} assumes greater significance. As a result, the method tends to predict a larger $\widetilde{K}^\star$, as in Table~\ref{tab:estimate-num-epochs}. By default, we choose to train 3 epochs for all datasets.

\subsection{Ablation Studies} \label{subsec:ablation}

\begin{table*}[!t]
    \setlength{\tabcolsep}{5pt}
    \centering
    \renewcommand{\arraystretch}{1.0}
    \caption{Main ablation studies on the learning objectives.}
    \vspace{-5pt}
    \label{tab:main-ablation}
    \resizebox{.8\textwidth}{!}{
    \begin{tabular}{@{}ccccccccccc@{}}
    \toprule
    \multirow{2}{*}{ID} & \multirow{2}{*}{\begin{tabular}[c]{@{}c@{}}Contrastive\\ $\mathcal{L}_\text{con}$\end{tabular}} & \multirow{2}{*}{\begin{tabular}[c]{@{}c@{}}DAPL\\ $\mathcal{L}_\text{dapl}$\end{tabular}} & \multirow{2}{*}{\begin{tabular}[c]{@{}c@{}}EntropyReg\\ $\mathcal{L}_\text{entropy}$\end{tabular}} & \multirow{2}{*}{\begin{tabular}[c]{@{}c@{}}ProtoSep\\ $\mathcal{L}_\text{sep}$\end{tabular}} & \multicolumn{3}{c}{CIFAR100} & \multicolumn{3}{c}{Aircraft} \\ \cmidrule(l){6-8} \cmidrule(l){9-11} 
     &  &  &  &  & All & Old & New & All & Old & New \\ \midrule
    (a) & \xmark & \xmark & \xmark & \xmark & 61.2 & 79.4 & 24.6 & 30.7 & 39.9 & 26.2 \\
    (b) & \cmark & \xmark & \xmark & \xmark & 64.0 & 73.4 & 45.3 & 33.6 & 33.1 & 33.9 \\
    (c) & \cmark & \cmark & \xmark & \xmark & 65.8 & 71.9 & 53.7 & 36.1 & 36.2 & 36.0 \\
    (d) & \cmark & \xmark & \cmark & \xmark & 30.1 & 44.0 & 2.2 & 20.8 & 42.5 & 10.0 \\
    {(e)} & \cmark & \cmark & \xmark & \cmark & {66.4} & {71.8} & {55.7} & {38.0} & {38.4} & {37.8} \\
    (f) & \cmark & \cmark & \cmark & \xmark & 79.7 & 79.6 & 79.9 & 54.4 & 56.6 & 53.3 \\
    (g) & \cmark & \cmark & \cmark & \cmark & \textbf{81.9} & \textbf{82.9} & \textbf{80.0} & \textbf{56.8} & \textbf{62.5} & \textbf{53.9} \\ \bottomrule
    \end{tabular}
    }
\end{table*}

\emph{Ablations on the main components.}
Here we validate the effectiveness of main training objectives, including contrastive learning $\mathcal{L}_\text{con}$ (Section~\ref{subsec:contrastive-learning}), DAPL mechanism $\mathcal{L}_\text{dapl}$ (Section~\ref{subsec:pseudo-label}), entropy regularization $\mathcal{L}_\text{entropy}$ (Section~\ref{subsubsec:regularization-trivial}) and separation regularization $\mathcal{L}_\text{sep}$ (Section~\ref{subsubsec:regularization-inter-class}). In Table~\ref{tab:main-ablation}, (a) is the baseline where only supervised classification $\mathcal{L}_\text{sup}$ in Eq.~\eqref{eq:loss-sup} is employed. (b) shows a slight improvement over (a), which implies that contrastive learning ensures fundamental representations. Comparing (b) and (c), DAPL improves overall performance, especially for `New' accuracy, highlighting the effectiveness of self-training with pseudo-labeling. Comparing (b) and (d), introducing $\mathcal{L}_\text{entropy}$ alone leads to collapsed performance. The reason is that blindly avoiding trivial solutions without the guidance of DAPL for pseudo-labeling brings about meaningless outcomes. In contrast, as in (f), the concurrent presence of DAPL and entropy regularization ensure significant performance gains, for instance, (f) outperforms (b) by $23.5\%$ and $19.4\%$ on `Old' and `New' classes of Aircraft, which highlights the importance of both DAPL mechanism and avoidance of trivial solutions in GCD. {In (e), removing $\mathcal{L}_\text{entropy}$ severely degrades the performance compared with (g) due to the trivial solutions.} Besides, explicitly separating clusters via $\mathcal{L}_\text{sep}$ further enhances the performance, with $2.2\%$ and $2.4\%$ improvements on two datasets. 

\begin{figure}[!t]
    \centering
    \begin{subfigure}{.48\linewidth}
        \centering
        \includegraphics[width=\linewidth]{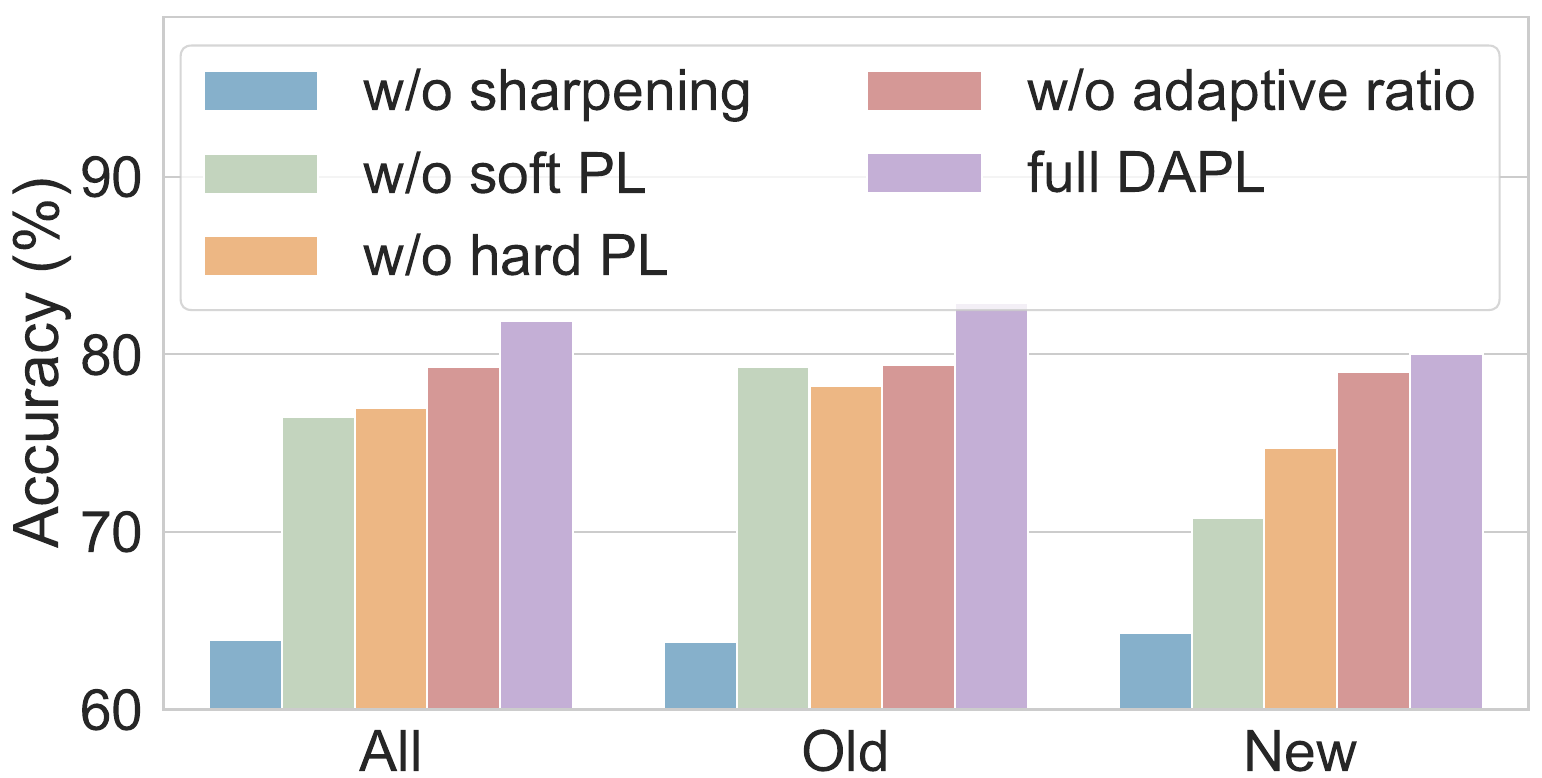}
        \vspace{-15pt}
        \caption{CIFAR100.}
        \label{subfig:ablation-dapl-cifar100}
    \end{subfigure}
    \hfill
    \begin{subfigure}{.48\linewidth}
        \centering
        \includegraphics[width=\linewidth]{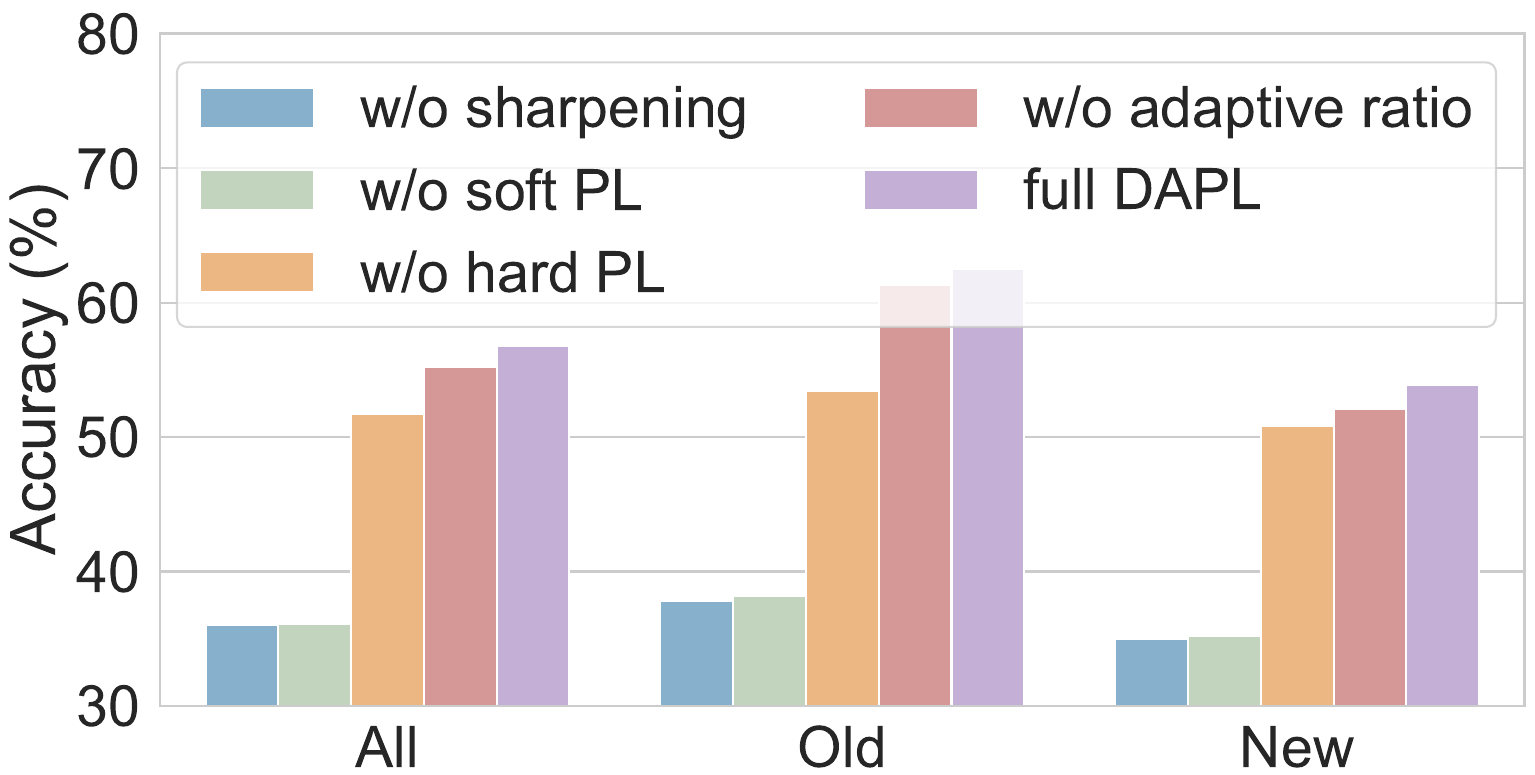}
        \vspace{-15pt}
        \caption{Aircraft.}
        \label{subfig:ablation-dapl-air}
    \end{subfigure}
    \vspace{-5pt}
    \caption{Detailed ablations on DAPL.}
    \label{fig:ablation-dapl}
    \vspace{-7pt}
\end{figure}

\emph{Detailed ablations on DAPL.}
We conduct ablations on our pseudo-labeling mechanism, including sharpening in soft pseudo-labels (PL), the combination of soft and hard PL and the adaptive ramp-up ratio of hard PL. In Fig.~\ref{fig:ablation-dapl}, the overall trends are similar across (a) and (b). Sharpening helps models produce more confident outputs, and the absence of sharpening impedes self-training, leading to significant performance decline, \ie, $\sim 20\%$. Models are susceptible to confirmation bias without soft PL, while without hard PL, the training is hindered due to less informative PL. Overall, soft PL has a more significant impact on the results. We also remove the adaptive ratio and fix the ratio of hard to soft PL at $1:1$, and the overall accuracy is roughly $\sim 2\%$ lower than the full DAPL, which underscores the importance of adaptivity according to the model's capabilities.

\begin{figure}[!t]
    \centering
    \includegraphics[width=.95\linewidth]{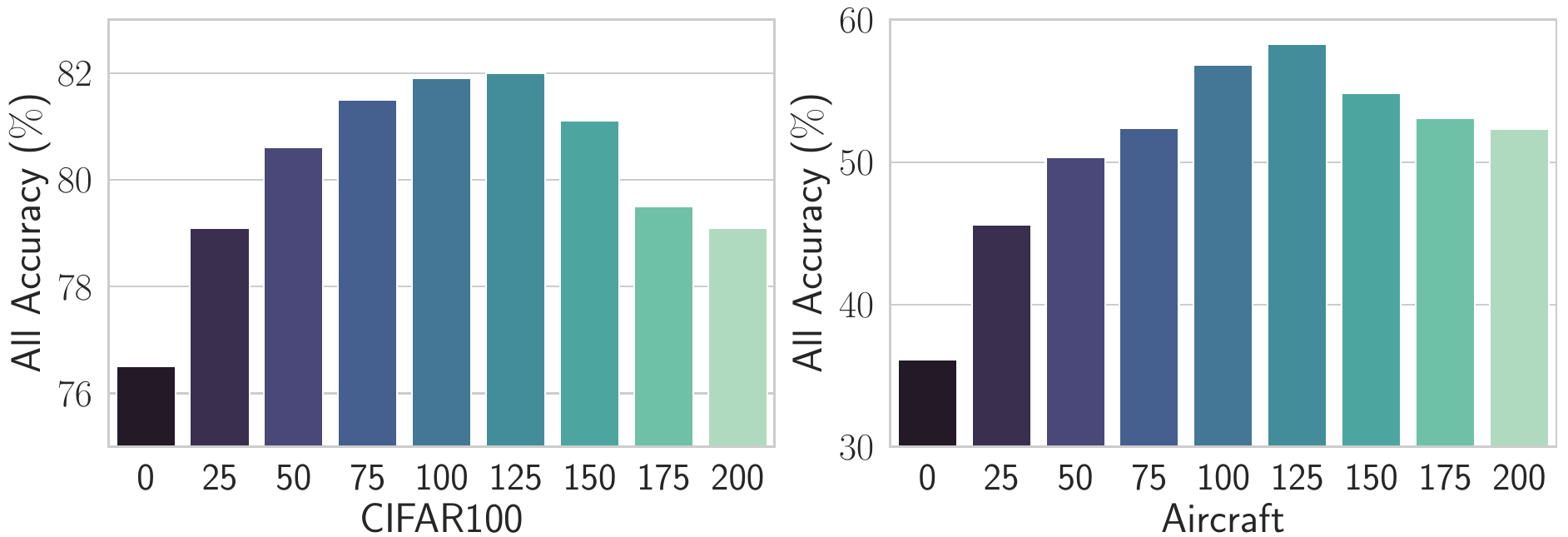}
    \vspace{-7pt}
    \caption{{Detailed ablations of $e_\text{ramp}$ on CIFAR100 and Aircraft.}}
    \label{fig:abaltion-ramp-epoch}
\end{figure}

\emph{Detailed ablations on $e_\text{ramp}$.}
{In the proposed DAPL, the proportion of samples assigned with hard pseudo-labels increases linearly from 0 to 100\% during the first $e_\text{ramp}$ epochs, as in Eq.~\eqref{eq:ratio-ramp-up}. Here, we conduct detailed ablation on the ramp-up epochs $e_\text{ramp}$ across $0,25,50,75,100,125,150,175,200$ on CIFAR100 and Aircraft. Results are shown in Fig.~\ref{fig:abaltion-ramp-epoch}. The optimal value of $e_\text{ramp}$ is around 125 for both datasets, and we could observe that the accuracy remains stable and high when $e_\text{ramp}$ ranges within $[75,150]$, showcasing the robustness of our method.}

\begin{figure}[!t]
    \centering
    \includegraphics[width=.7\linewidth]{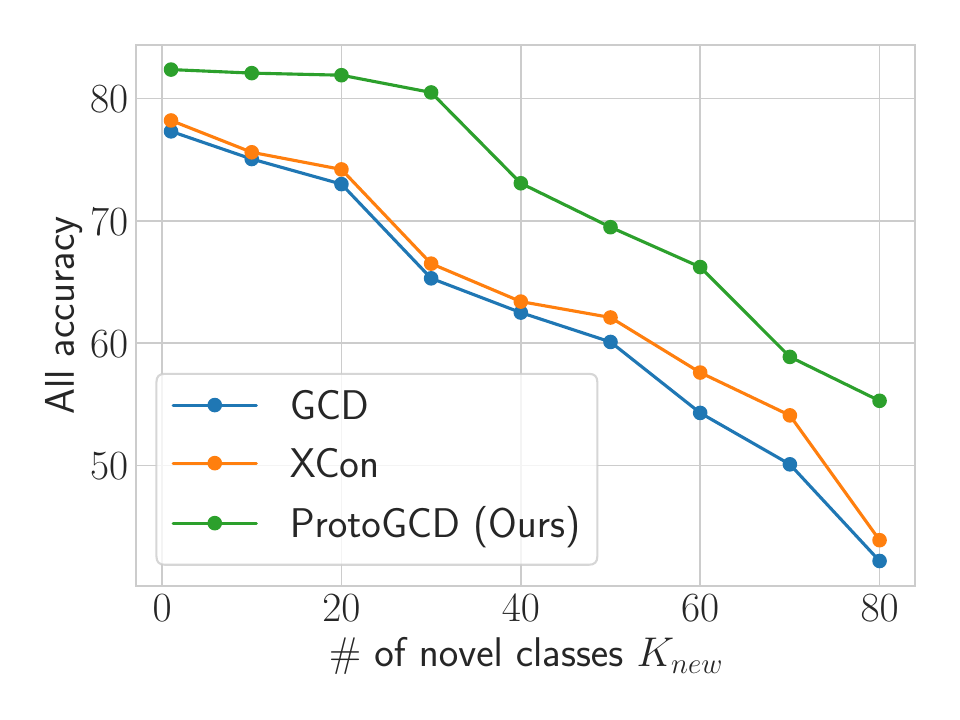}
    \vspace{-5pt}
    \caption{`All' accuracy across various class splits.}
    \label{fig:splits}
\end{figure}

\emph{Evaluation with various old/new class splits.}
We further evaluate different methods across various class splits on CIFAR100 where $K_{new}$ ranges from 1 to 99. Fig.~\ref{fig:splits} illustrates `All' accuracy and indicates that our method is more robust when very few classes are labeled, and consistently outperforms the competitors across various class splits.

\begin{figure}[!t]
    \centering
    \includegraphics[width=.8\linewidth]{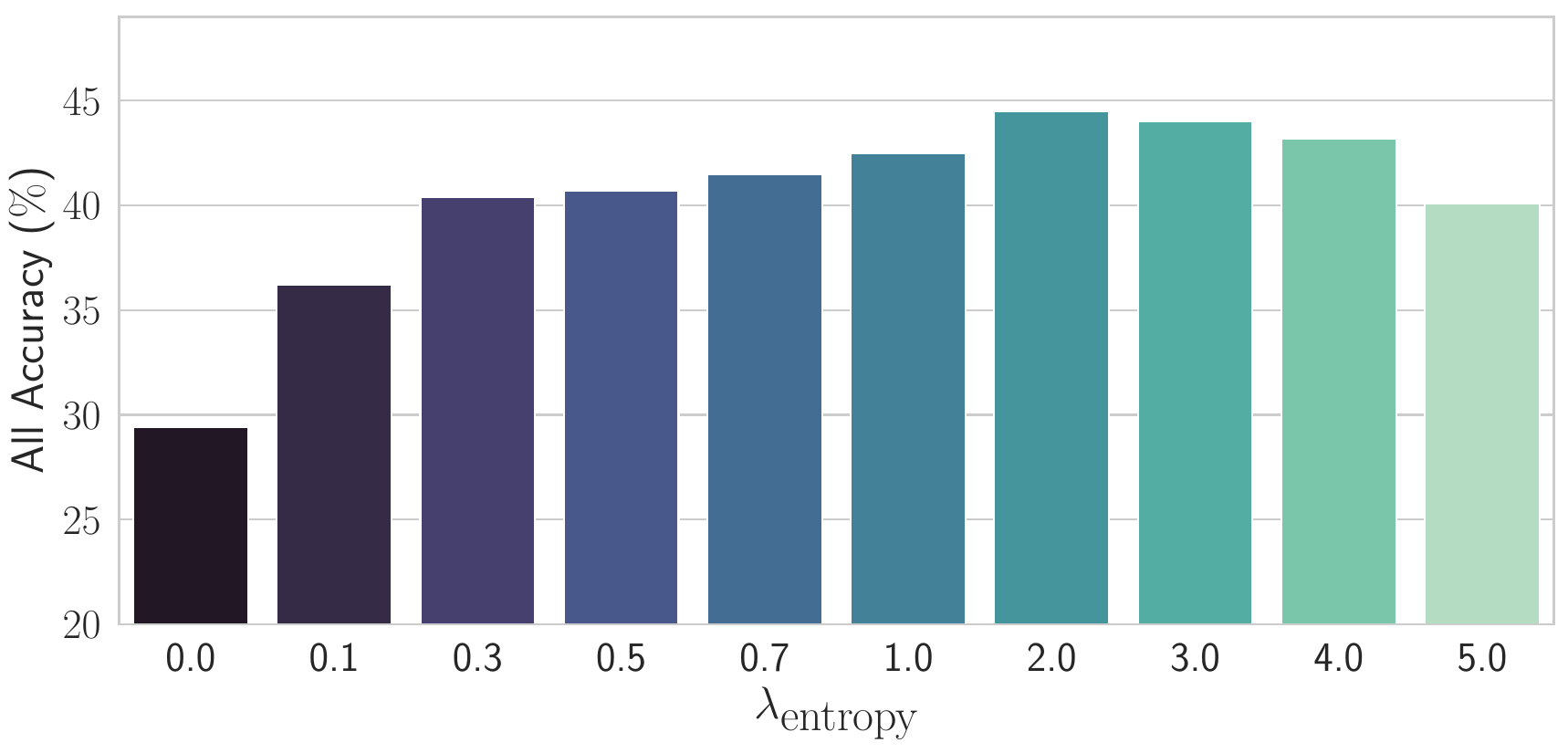}
    \vspace{-7pt}
    \caption{Detailed ablations of $\lambda_\text{entropy}$ on Herb19.}
    \label{fig:entropy-weight-herb}
    \vspace{-5pt}
\end{figure}

\emph{In-depth analysis of $\mathcal{L}_\text{entropy}$ on Herb dataset.}
{Although the entropy regularization $\mathcal{L}_\text{entropy}$ has implicitly imposed the assumption of a uniform distribution on the dataset, which might conflict with the long-tailed Herb, we conduct detailed ablations and argue that $\mathcal{L}_\text{entropy}$ is still a relatively applicable regularization in GCD. As Fig.~\ref{fig:entropy-weight-herb} shows, the results indicate a huge degradation in the absence of marginal entropy maximization $\mathcal{L}_\text{entropy}$ (44.5\% $\to$ 29.4\%). Even when imposing a small weight, \eg, 0.1, there is a notable enhancement (29.4\% $\to$ 36.2\%). In summary, $\mathcal{L}_\text{entropy}$ is indispensable. The reason is that (1) $\mathcal{L}_\text{entropy}$ is a soft regularization rather than the rigid constraints like the equipartition in~\cite{Fini_2021_ICCV}, we could choose appropriate $\lambda_\text{entropy}$ to balance between avoiding trivial solutions and preventing conflicts with the actual dataset distribution. (2) $\mathcal{L}_\text{entropy}$ directly acts on the model's predicted marginal probabilities, which may not strictly align with the ratio of samples from new and old classes predicted by the model. The latter corresponds to the actual distribution of the dataset. More details are shown in the Appendix.}

\subsection{OOD Detection Performance}

In this section, we extend ProtoGCD to OOD detection scenarios, as described in Section~\ref{sec:extend-to-ood}, and compare its rejection ability of unseen classes with GCD methods~\cite{vaze2022gcd,fei2022xcon}.

\emph{Experimental Setup.}
For CIFAR100 as ID dataset, test OOD datasets are Texture~\cite{cimpoi2014describing}, SVHN~\cite{37648}, Places365~\cite{zhou2017places}, TinyImageNet, LSUN~\cite{yu15lsun}, iSUN~\cite{xu2015turkergaze} and CIFAR10. For ImageNet-100 as ID dataset, test OOD datasets are Texture~\cite{cimpoi2014describing}, Places365~\cite{zhou2017places}, iNaturalist~\cite{van2018inaturalist}, ImageNet-O~\cite{hendrycks2021natural}, OpenImage-O~\cite{wang2022vim}. Following the convention~\cite{hendrycks2016baseline,pmlr-v162-hendrycks22a}, we use AUROC and FPR95 to measure OOD detection. We treat ID classes ($\mathcal{C}_{old}\cup\mathcal{C}_{new}$) as positives, and OOD classes ($\mathcal{C}_{out}$) as negatives. More details are shown in the Appendix.

\begin{table}[!t]
    \setlength{\tabcolsep}{4pt}
    \centering
    \renewcommand{\arraystretch}{1.0}
    \caption{OOD detection performance on CIFAR100.}
    \vspace{-5pt}
    \label{tab:test-ood-cifar100}
    \resizebox{.95\linewidth}{!}{
    \begin{tabular}{@{}cccccccccc@{}}
    \toprule
    \multirow{2}{*}{$\mathcal{D}_\textrm{out}^\textrm{test}$} & \multicolumn{3}{c}{FPR95 $\downarrow$} & \multicolumn{3}{c}{AUROC $\uparrow$} \\ \cmidrule(l){2-4} \cmidrule(l){5-7} 
     & GCD & XCon & ProtoGCD & GCD & XCon & ProtoGCD \\ \midrule
    Texture & \textbf{29.92} & 42.81 & 31.31 & 92.99 & 90.74 & \textbf{93.90} \\
    SVHN & \textbf{47.80} & 51.54 & 50.65 & 90.54 & 90.89 & \textbf{91.21} \\
    Places365 & \textbf{49.36} & 69.20 & 56.17 & \textbf{86.68} & 81.31 & 84.20 \\
    TinyImageNet & 59.08 & 60.88 & \textbf{58.93} & 84.62 & 84.00 & \textbf{85.94} \\
    LSUN & 71.16 & 63.40 & \textbf{60.89} & 83.42 & 84.68 & \textbf{87.07} \\
    iSUN & 69.15 & 65.0 & \textbf{64.03} & 82.87 & 83.97 & \textbf{84.51} \\
    CIFAR10 & 71.97 & 68.53 & \textbf{63.53} & 77.59 & 78.13 & \textbf{80.18} \\ \midrule
    Mean & 56.92 & 60.20 & \textbf{55.07} & 85.53 & 84.82 & \textbf{86.72} \\ \bottomrule
    \end{tabular}
    }
\end{table}

\begin{table}[!t]
    \setlength{\tabcolsep}{4pt}
    \centering
    \renewcommand{\arraystretch}{1.0}
    \caption{OOD detection performance on ImageNet-100.}
    \vspace{-5pt}
    \label{tab:test-ood-imagenet}
    \resizebox{.95\linewidth}{!}{
    \begin{tabular}{@{}cccccccccc@{}}
    \toprule
    \multirow{2}{*}{$\mathcal{D}_\textrm{out}^\textrm{test}$} & \multicolumn{3}{c}{FPR95 $\downarrow$} & \multicolumn{3}{c}{AUROC $\uparrow$} \\ \cmidrule(l){2-4} \cmidrule(l){5-7} 
     & GCD & XCon & ProtoGCD & GCD & XCon & ProtoGCD \\ \midrule
    Texture & 46.62 & 39.79 & \textbf{21.75} & 91.60 & 93.70 & \textbf{94.60} \\
    Places365 & 66.37 & 67.82 & \textbf{56.47} & \textbf{87.00} & 86.88 & 85.09 \\
    iNaturalist & 70.30 & 69.87 & \textbf{52.29} & 86.28 & 86.29 & \textbf{87.72} \\
    ImageNet-O & 63.47 & 61.70 & \textbf{48.91} & 85.75 & 87.23 & \textbf{87.89} \\
    OpenImage-O & 64.34 & 60.84 & \textbf{46.64} & 86.56 & 88.43 & \textbf{89.45} \\ \midrule
    Mean & 62.22 & 60.00 & \textbf{45.21} & 87.44 & 88.51 & \textbf{88.95} \\ \bottomrule
    \end{tabular}
    }
\end{table}

\emph{Comparative Results.}
As discussed in Section~\ref{sec:extend-to-ood}, ProtoGCD could obtain posterior probabilities with the learned prototypes (Eq.~\eqref{eq:posterior-prob}), for non-parametric methods~\cite{vaze2022gcd,fei2022xcon}, we firstly run K-means on the training set of GCD and employ the cluster centroids of $\mathcal{C}_{old}\cup\mathcal{C}_{new}$ to get predictive probabilities. For fair comparisons, we use MSP~\cite{hendrycks2016baseline} as the score function, and conduct OOD detection on CIFAR100 (Table~\ref{tab:test-ood-cifar100}) and ImageNet-100 (Table~\ref{tab:test-ood-imagenet}). ProtoGCD demonstrates stronger OOD detection capability, \eg, on CIFAR100, it achieves $1.85\%$ lower FPR95 and $1.19\%$ higher AUROC compared to GCD~\cite{vaze2022gcd}.

\begin{table}[!tb]
    \setlength{\tabcolsep}{4pt}
    \centering
    \renewcommand{\arraystretch}{0.8}
    \caption{OOD detection performance with other scores.}
    \vspace{-5pt}
    \label{tab:ood-score}
    \resizebox{.9\linewidth}{!}{
    \begin{tabular}{@{}ccccc@{}}
    \toprule
    \multirow{2}{*}{OOD Scores} & \multicolumn{2}{c}{CIFAR100} & \multicolumn{2}{c}{ImageNet-100} \\ \cmidrule(l){2-3} \cmidrule(l){4-5} 
     & FPR95 $\downarrow$ & AUROC $\uparrow$ & FPR95 $\downarrow$ & AUROC $\uparrow$ \\ \midrule
    MSP~\cite{hendrycks2016baseline} & 55.07 & 86.72 & 45.21 & 88.95 \\
    MLS~\cite{pmlr-v162-hendrycks22a} & \textbf{54.90} & 86.85 & 44.40 & 89.24 \\
    Energy~\cite{liu2020energy} & 54.97 & \textbf{88.57} & \textbf{42.74} & \textbf{90.07} \\ \bottomrule
    \end{tabular}
    }
\end{table}

\begin{figure*}[!t]
    \centering
    \begin{subfigure}{.47\linewidth}
        \centering
        \includegraphics[width=\linewidth]{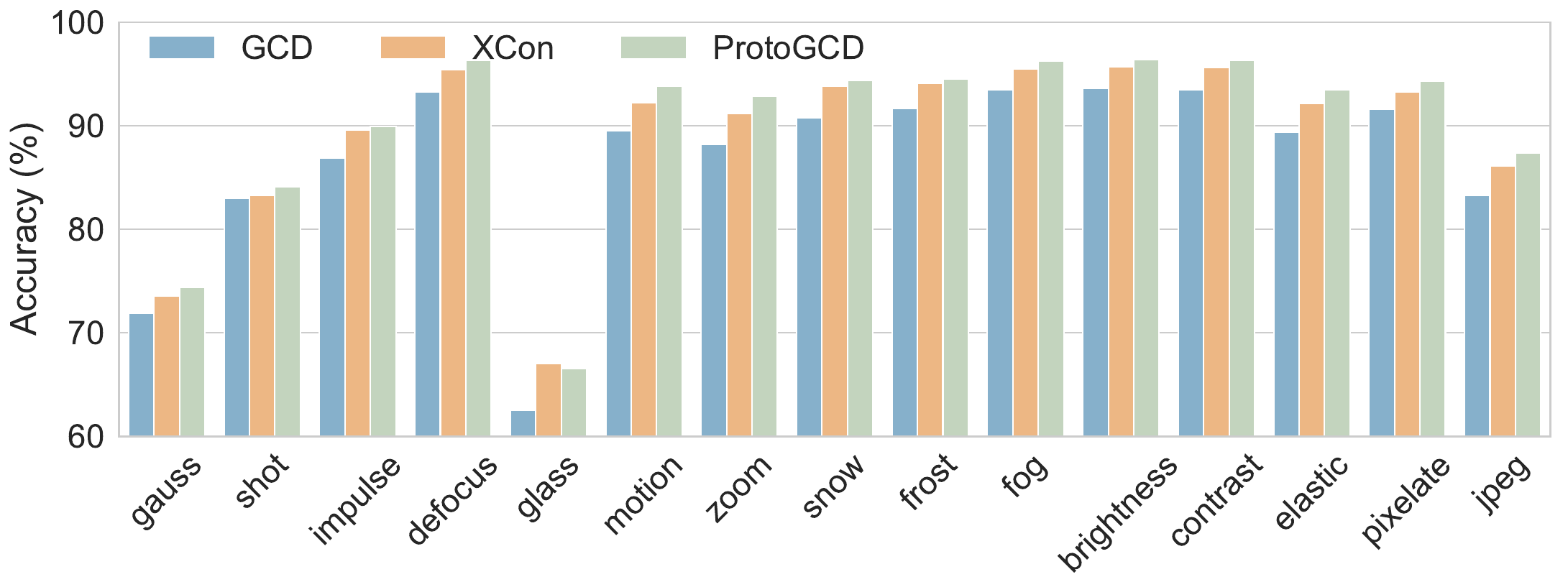}
        \vspace{-15pt}
        \caption{CIFAR-10-C.}
        \label{subfig:corruption-cifar10c}
    \end{subfigure}
    \begin{subfigure}{.47\linewidth}
        \centering
        \includegraphics[width=\linewidth]{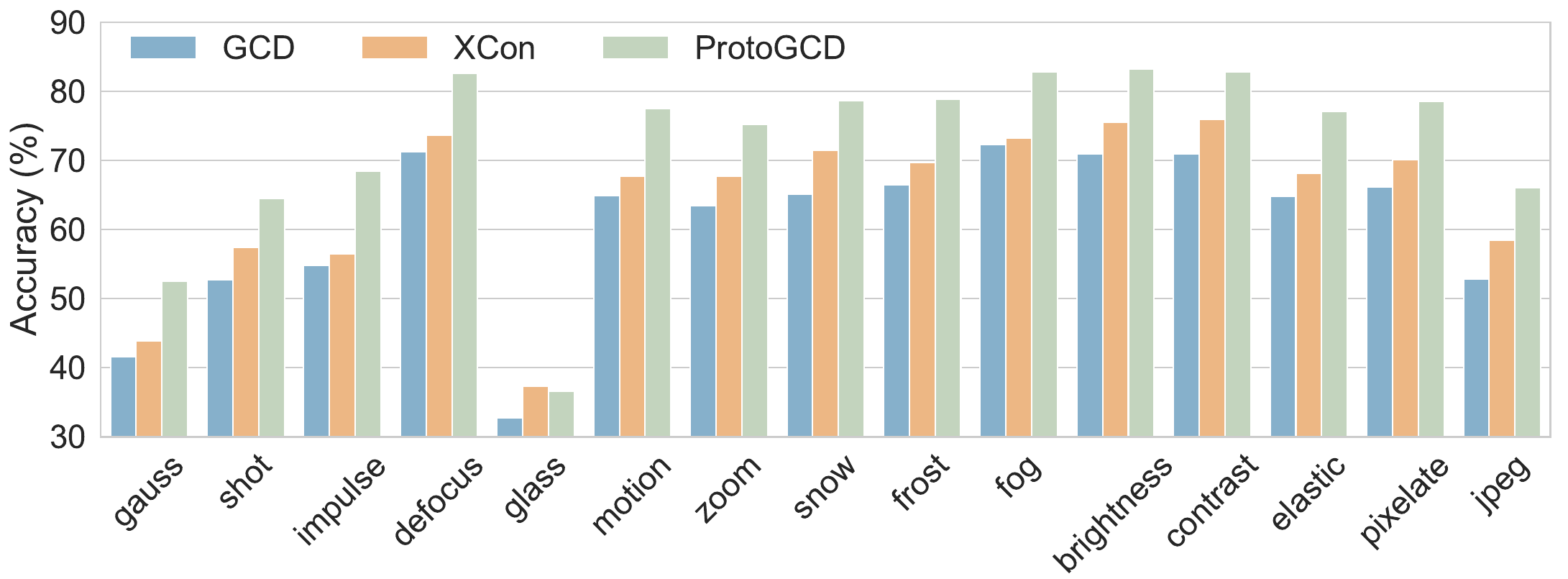}
        \vspace{-15pt}
        \caption{CIFAR-100-C.}
        \label{subfig:corruption-cifar100c}
    \end{subfigure}
    \vspace{-5pt}
    \caption{`All' accuracy (\%) on distribution shift scenarios. CIFAR-10-C and CIFAR-100-C are corruption datasets that contain 15 types of corruption, each with 5 levels of severity. Results here are at severity 1.}
    \label{fig:corruption}
    \vspace{-5pt}
\end{figure*}

\emph{OOD Detection with Other Score Functions.}
We also conduct OOD detection of ProtoGCD under different OOD scores, including max logit score (MLS)~\cite{pmlr-v162-hendrycks22a} and Energy~\cite{liu2020energy}. MLS explores logits in the feature space, namely the similarity to prototypes, $\max_c\boldsymbol{\mu}_c^\top \mathbf{z}_i$, while Energy aligns better with the density of data~\cite{liu2020energy}. Consequently, MLS and Energy outperform the MSP baseline, as validated in Table~\ref{tab:ood-score}.

\subsection{Further Analysis}

\subsubsection{Category Discovery under Covariate-Shifts}
Existing GCD works predominantly assume the data distribution is invariant. However, samples inevitably undergo covariate-shifts~\cite{9080115} in the ever-changing environments. The model is still required to robustly discover distribution-shifted novel categories. In this paper, we evaluate the performance of GCD~\cite{vaze2022gcd}, XCon~\cite{fei2022xcon} and our ProtoGCD under distribution shifts.

\emph{Experimental Setup.}
We directly use models trained in standard GCD settings, \ie, CIFAR10/100, which are then evaluated on the corrupted datasets, \ie, CIFAR10/100-C~\cite{hendrycks2018benchmarking}. It is worth noting that in corrupted test data, there are only covariate-shifts without semantic-shifts. The evaluation dataset contains 15 types of synthetic corruptions, with 5 levels of severity for each, resulting in 75 distinct corruptions. The corruptions include noise, weather changes and digital operations (see the horizontal axis of Fig.~\ref{fig:corruption}).

\begin{table}[!t]
    \setlength\tabcolsep{4pt}
    \centering
    \renewcommand{\arraystretch}{.8}
    \caption{Performance degradation under different levels of corruption severity on CIFAR100-C of our method.}
    \vspace{-5pt}
    \label{tab:cifar100c-gap}
    \resizebox{.8\linewidth}{!}{
    \begin{tabular}{@{}cccc@{}}
    \toprule
    Level & All & Old & New \\ \midrule
    0 & 81.9 & 82.9 & 80.0 \\
    1 & 72.4 $\color{Red}(~~9.5\downarrow)$ & 73.4 $\color{Red}(~~9.5\downarrow)$ & 68.2 $\color{Red}(11.8\downarrow)$ \\
    2 & 66.0 $\color{Red}(15.9\downarrow)$ & 66.9 $\color{Red}(16.0\downarrow)$ & 62.4 $\color{Red}(17.6\downarrow)$ \\
    3 & 60.0 $\color{Red}(21.9\downarrow)$ & 60.7 $\color{Red}(22.2\downarrow)$ & 57.1 $\color{Red}(22.9\downarrow)$ \\
    4 & 53.8 $\color{Red}(28.1\downarrow)$ & 54.3 $\color{Red}(28.6\downarrow)$ & 51.7 $\color{Red}(28.3\downarrow)$ \\
    5 & 43.4 $\color{Red}(38.5\downarrow)$ & 43.8 $\color{Red}(39.1\downarrow)$ & 41.8 $\color{Red}(38.2\downarrow)$ \\ \bottomrule
    \end{tabular}
    }
\end{table}

\emph{Experimental Results.}
Comparative results at severity 1 of three methods are shown in Fig.~\ref{fig:corruption}. ProtoGCD consistently outperforms GCD and XCon, for instance, regarding \texttt{snow}  and \texttt{jpeg} of CIFAR100-C, our method achieves $7.19\%$ and $7.60\%$ higher `All' accuracy over XCon. Besides, we implement our methods on 5 levels of severity. As Table~\ref{tab:cifar100c-gap} reveals, at lower levels of severity, performance degradation for new classes is more significant than for old classes, but at higher levels of severity, the decrease is similar for both.

\subsubsection{Cluster Characteristics}

To further quantitatively evaluate the learned feature representations of GCD, we present the following two metrics of intra-class compactness and inter-class separation:
\begin{align}
    \mathsf{compactness} \uparrow & = \frac{1}{K}\sum_{k=1}^K \frac{1}{|\mathcal{D}_\text{test}^k|}\sum_{i\in\mathcal{D}_\text{test}^k} \overline{\boldsymbol{\mu}}_k^\top\mathbf{z}_i, \label{eq:cluster-compactness} \\
    \mathsf{separation} \downarrow & = \frac{1}{K}\sum_{i=1}^K\frac{1}{K-1}\sum_{j=1,j\neq i}^K \overline{\boldsymbol{\mu}}_i^\top\overline{\boldsymbol{\mu}}_j, \label{eq:cluster-separation}
\end{align}
where $K=K_{old}+K_{new}$ denotes total number of classes, $\mathcal{D}_\text{test}^k=\{(\mathbf{x}_i,y_i)\in\mathcal{D}_\text{test},y_i=k\}$ is the $i$-th classes of the test dataset, $\overline{\boldsymbol{\mu}}_k=\frac{1}{|\mathcal{D}_\text{test}^k|}\sum_{i\in\mathcal{D}_\text{test}^k} \mathbf{z}_i$ is the $\ell_2$-normalized mean feature of class $k$. Due to the cosine similarity in Eq.~\eqref{eq:cluster-compactness} and Eq.~\eqref{eq:cluster-separation}, greater intra-class compactness and inter-class separation lead to higher $\mathsf{compactness}$ and lower $\mathsf{separation}$.

\newcommand{\clustercifar}{
    \begin{tabular}{@{}ccccc@{}}
    \toprule
    \multirow{2}{*}{Methods} & \multicolumn{3}{c}{$\mathsf{Compactness} \uparrow$} & \multirow{2}{*}{$\mathsf{Separation} \downarrow$} \\ \cmidrule(lr){2-4}
     & All & Old & New &  \\ \midrule
    GCD & 0.66 & 0.67 & 0.63 & 0.20 \\
    XCon & 0.71 & 0.72 & 0.67 & 0.13 \\
    {DCCL} & {0.75} & {0.76} & {0.70} & {0.11} \\
    {GPC} & {0.70} & {0.71} & {0.66} & {0.15} \\
    \RC{30}Ours & \textbf{0.80} & \textbf{0.79} & \textbf{0.81} & \textbf{0.01} \\ \bottomrule
    \end{tabular}
}

\newcommand{\clustercub}{
    \begin{tabular}{@{}ccccc@{}}
    \toprule
    \multirow{2}{*}{Methods} & \multicolumn{3}{c}{$\mathsf{Compactness} \uparrow$} & \multirow{2}{*}{$\mathsf{Separation} \downarrow$} \\ \cmidrule(lr){2-4}
     & All & Old & New &  \\ \midrule
    GCD & 0.76 & 0.78 & 0.75 & 0.17 \\
    XCon & 0.77 & 0.78 & 0.77 & 0.17 \\
    {DCCL} & {0.78} & {0.79} & {0.78} & {0.13} \\
    {GPC} & {0.75} & {0.76} & {0.73} & {0.15} \\
    \RC{30}Ours & \textbf{0.79} & \textbf{0.80} & \textbf{0.79} & \textbf{0.12} \\ \bottomrule
    \end{tabular}
}

\begin{table}[!t]
    \setlength\tabcolsep{2pt}
    \centering
    \caption{Cluster metrics (intra-class compactness $\uparrow$ and inter-class separation $\downarrow$) on CIFAR100 (a) and CUB (b).}
    \vspace{-5pt}
    \renewcommand{\arraystretch}{1.0}
    \begin{subtable}{0.49\linewidth}
        \resizebox{\linewidth}{!}{\clustercifar}
        \vspace{1pt}
        \caption{CIFAR100.}
        \label{tab:cluster-analysis-cifar100}
    \end{subtable}\hfill
    \begin{subtable}{0.49\linewidth}
        \resizebox{\linewidth}{!}{\clustercub}
        \vspace{1pt}
        \caption{CUB.}
        \label{tab:cluster-analysis-cub}
    \end{subtable}
    \label{tab:cluster-methods}
    \vspace{-5pt}
\end{table}

\begin{figure}[!tb]
    \centering
    \includegraphics[width=0.85\linewidth]{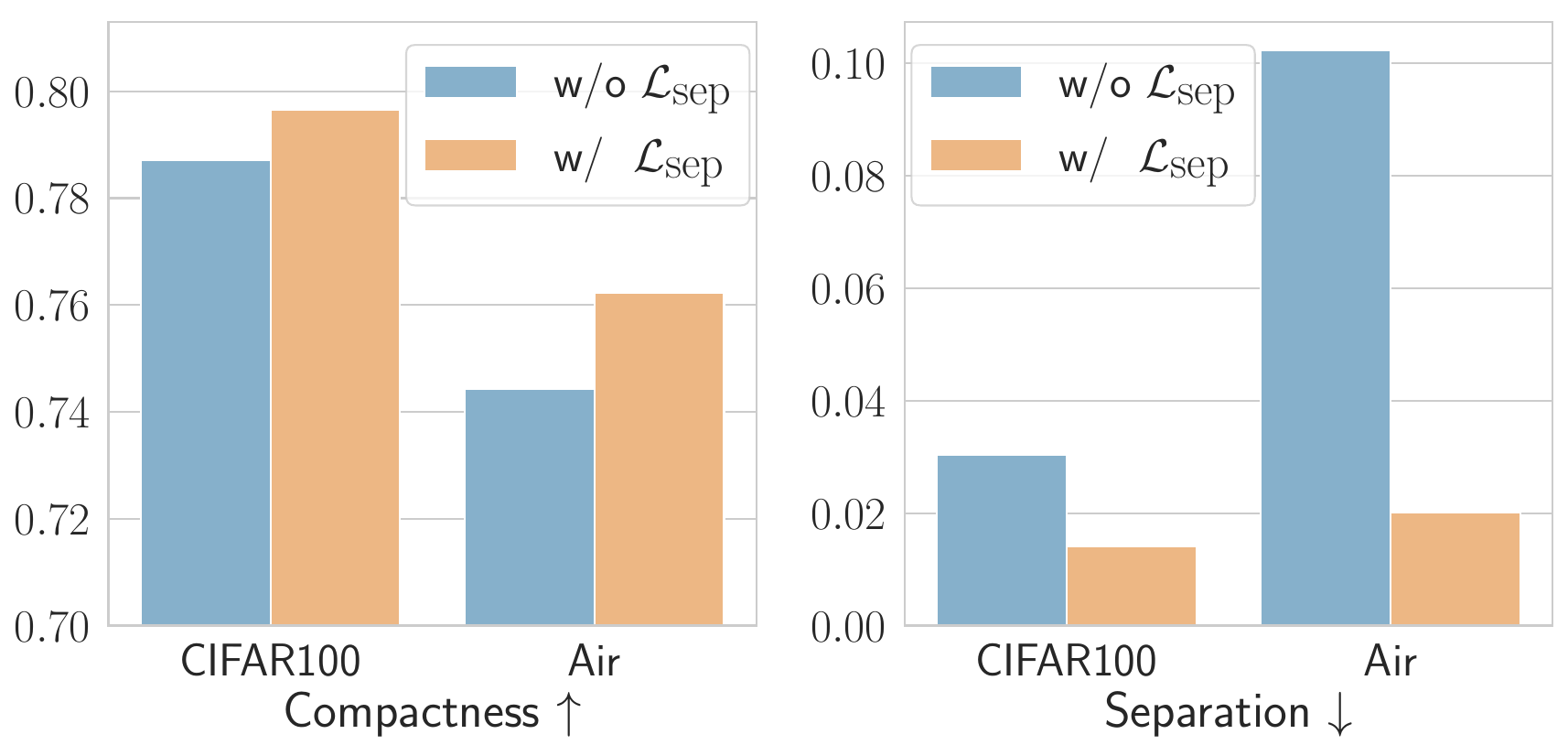}
    \vspace{-5pt}
    \caption{Cluster metrics w/ and w/o prototype separation loss $\mathcal{L}_\text{sep}$, including intra-class compactness $\uparrow$ (left) and inter-class separation $\downarrow$ (right).}
    \label{fig:cluster-analysis-losssep}
\end{figure}

We compute $\mathsf{compactness}$ and $\mathsf{separation}$ in the test dataset, as shown in Table~\ref{tab:cluster-methods}. Regarding the two metrics, ProtoGCD outperforms its competitors on both CIFAR100 and CUB. Take CUB as an instance, ProtoGCD improves $\mathsf{compactness}$ from $0.77$ of XCon~\cite{fei2022xcon} to $0.79$, and decreases $\mathsf{separation}$ from $0.17$ to $0.12$. The results demonstrate that ProtoGCD learns better representations with greater intra-class compactness and inter-class separation.
We further validate the effectiveness of prototype separation loss. As Fig.~\ref{fig:cluster-analysis-losssep} shows, while $\mathcal{L}_\text{sep}$ explicitly enhances inter-class separation, it also implicitly increases intra-class compactness.

\begin{figure*}[!tb]
    \centering
    \includegraphics[width=.9\textwidth]{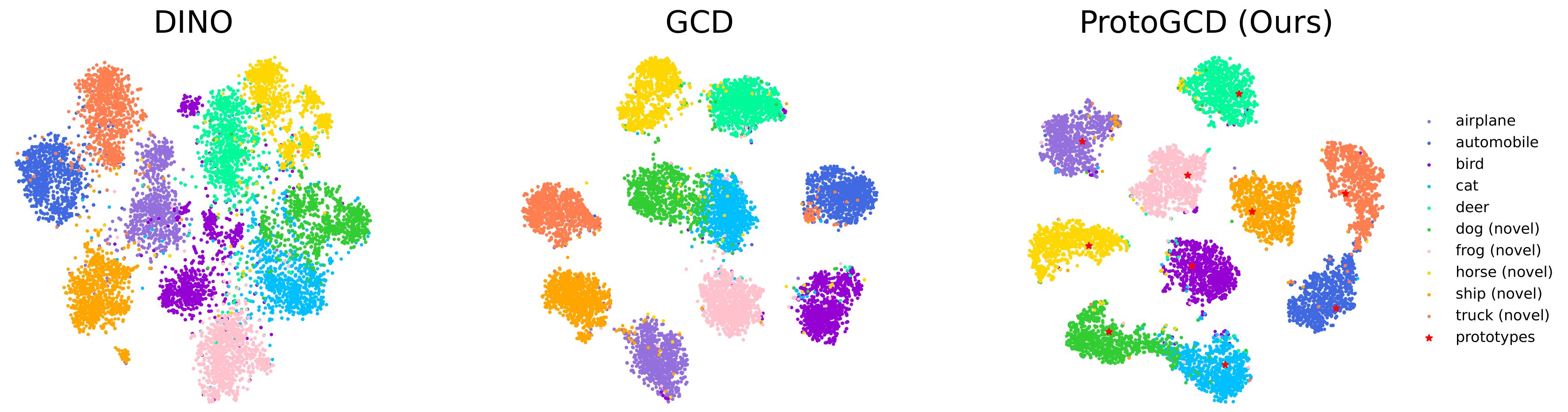}
    \vspace{-5pt}
    \caption{Visualizations of the feature space on CIFAR10. Features of old classes are depicted in cool colors (\eg, $\color{Orchid}{\bullet}$, $\color{blue}{\bullet}$, $\color{cyan}{\bullet}$, $\color{SeaGreen}{\bullet}$) while novel categories in warm colors (\eg, $\color{Yellow}{\bullet}$, $\color{orange}{\bullet}$, $\color{pink}{\bullet}$, $\color{YellowGreen}{\bullet}$). Additionally, the learnable prototypes are denoted as ${\color{red}{\star}}$. Our method provides improved inter-class separation and intra-class compactness.}
    \label{fig:vis-tsne}
\end{figure*}

\subsection{Qualitative Visualization and Analysis}
In this section, we provide visualizations of the feature space (Section~\ref{subsubsec:vis-tsne}) and the attention map (Section~\ref{subsubsec:vis-attention}) to qualitatively verify the effectiveness and superiority of our method. Our method could also retrieve samples with prototypes (Section~\ref{subsubsec:sample-retrieval}).

\subsubsection{Visualizations of the Feature Space} \label{subsubsec:vis-tsne}
We first show feature space visualizations of three methods: pre-trained DINO~\cite{caron2021emerging}, the classical approach GCD~\cite{vaze2022gcd} and our ProtoGCD using t-SNE~\cite{van2008visualizing}. Visualizations on CIFAR10~\cite{krizhevsky2009learning} are illustrated in Fig.~\ref{fig:vis-tsne}.

\emph{ProtoGCD improves intra-class compactness and effectively helps mitigate confirmation bias.} The DAPL module in Eq.~\eqref{eq:adaptive-pl} with the parametric prototypical classifier gradually assigns high-quality pseudo-labels, encouraging samples to move toward their associated prototypes, leading to compact clusters. By contrast, GCD~\cite{vaze2022gcd} with a non-parametric classifier resorts to pure contrastive learning, which performs instance discrimination~\cite{pmlr-v119-chen20j} and treats any two samples as negative pairs, even if they belong to the same class. As a result, it suffers from the \textit{class collision} issue, resulting in dispersed and sparse clusters. For example, in Fig.~\ref{fig:vis-tsne}, the clusters of \texttt{ship} and \texttt{horse} in GCD are dispersed, while in our methods are more compact. Overall, the class-wise prototypes help place each cluster in reasonable locations, and ProtoGCD benefits from the synergy of the proposed pseudo-labeling mechanism and learning objectives.

\emph{ProtoGCD further improves inter-class separation.}
The prototype separation loss $\mathcal{L}_\text{sep}$ explicitly pushes the prototypes far away from each other, which improves inter-class separation. As Fig.~\ref{fig:vis-tsne} shows, \texttt{deer} and \texttt{horse} in GCD~\cite{vaze2022gcd} tend to overlap and become intertwined, posing challenges to distinguishing among them, while our method achieves clear cluster boundaries and separated clusters.

\subsubsection{Visualizations of the Attention Map} \label{subsubsec:vis-attention}
We visualize the attention mechanism of the ViT backbone pre-trained with DINO, fine-tuned with Xcon~\cite{fei2022xcon} and our ProtoGCD in Fig.~\ref{fig:vis-attention}. Specifically, self-attention maps of $\texttt{[CLS]}$ token over three heads in the last layer are displayed. We conduct experiments on Stanford Cars~\cite{krause20133d} and CUB~\cite{wah2011caltech}. The regions with the top attention values are highlighted in red, and deeper red indicates higher attention values.

\begin{figure*}[!tb]
    \centering
    \includegraphics[width=.85\textwidth]{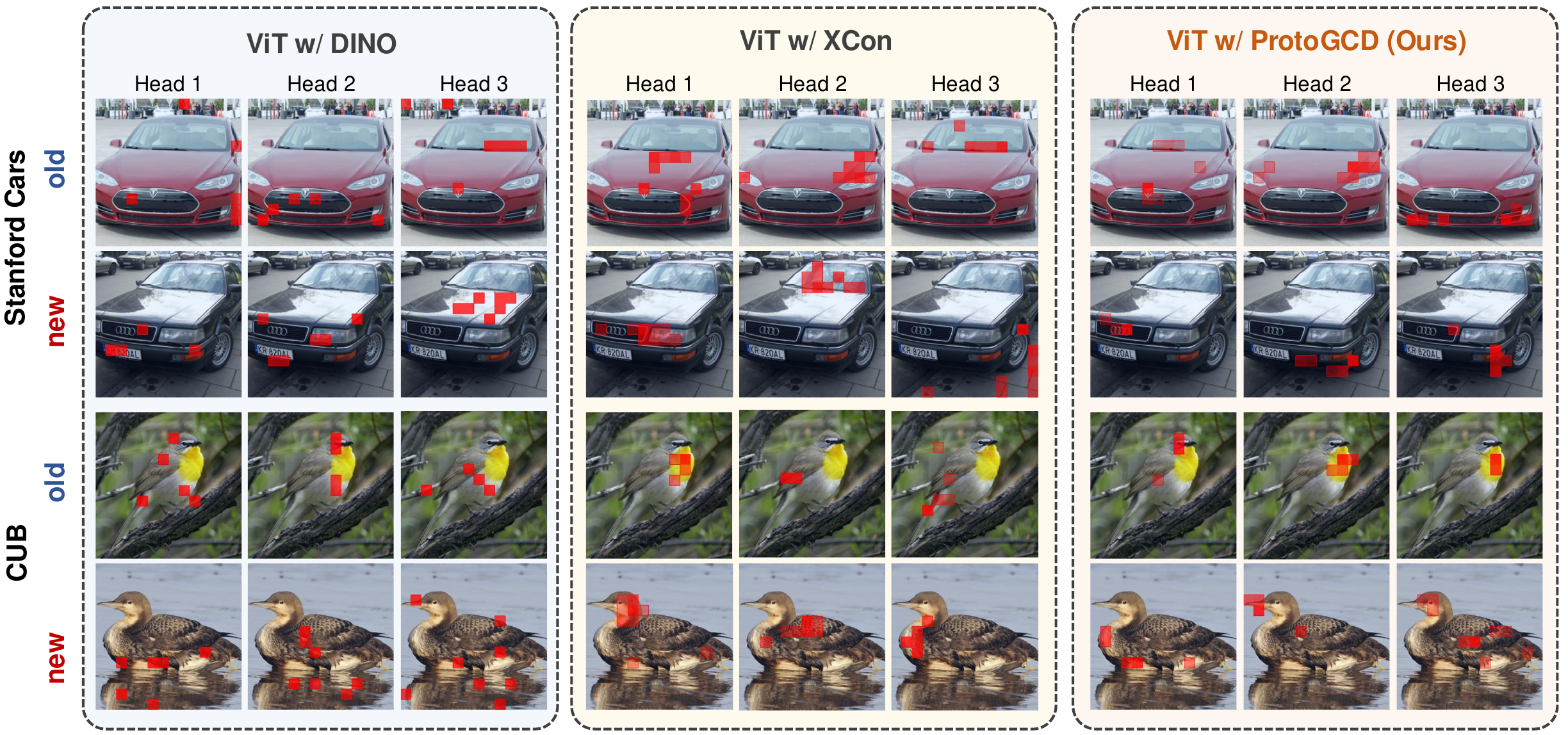}
    \vspace{-5pt}
    \caption{{Visualizations attention maps. For Stanford Cars (top), \texttt{Tesla} (old) and \texttt{Audi} (new) are shown. For CUB (bottom), \texttt{Yellow\_Breasted\_Chat} (old) and \texttt{Pacific\_Loon} (new) are displayed. Please zoom in for more details.}}
    \label{fig:vis-attention}
\end{figure*}

\emph{ProtoGCD produces attended regions with greater concentration and alleviates the spurious correlation regions.}
Overall, in Fig.~\ref{fig:vis-attention}, the attention maps of the pre-trained DINO are relatively sparse and dispersed. For instance, attended regions of \texttt{Pacific\_Loon} distribute across different locations. Even worse, DINO attends to spurious correlation background areas, like surroundings near the car (head 1 and 3 of \texttt{Tesla}), tree branches (head 1 of \texttt{Yellow\_Breasted\_Chat}) and water (three heads of \texttt{Pacific\_Loon}). By contrast, ProtoGCD could greatly mitigate the spurious correlation and focus on core regions to discern classes in a fine-grained manner, like cars' logo (head 1 of \texttt{Tesla} and \texttt{Audi}) and birds' eyes (head 1 of \texttt{Yellow\_Breasted\_Chat}) and beaks (head 2 of \texttt{Pacific\_Loon}). Additionally, the attention areas of each head in ProtoGCD are more concentrated and precise.

\emph{ProtoGCD effectively transfers the classification capabilities from old classes to novel categories.}
In GCD, models are expected to learn the classification criterion, \ie, what constitutes a class and how to discern different classes, on labeled classes, and transfer the knowledge to novel categories. The results in Fig.~\ref{fig:vis-attention} effectively substantiate this point. Specifically, car logos are one of the most salient areas for car classification. Models learn to attend to the car logo on \texttt{Tesla} (head 1 of ProtoGCD) from old classes, and manage to attend to car logos of the new class \texttt{Audi} (head 1). One could also observe a similar phenomenon in birds' eyes in Fig.~\ref{fig:vis-attention}.

\begin{figure}[!t]
    \centering
    \includegraphics[width=.95\linewidth]{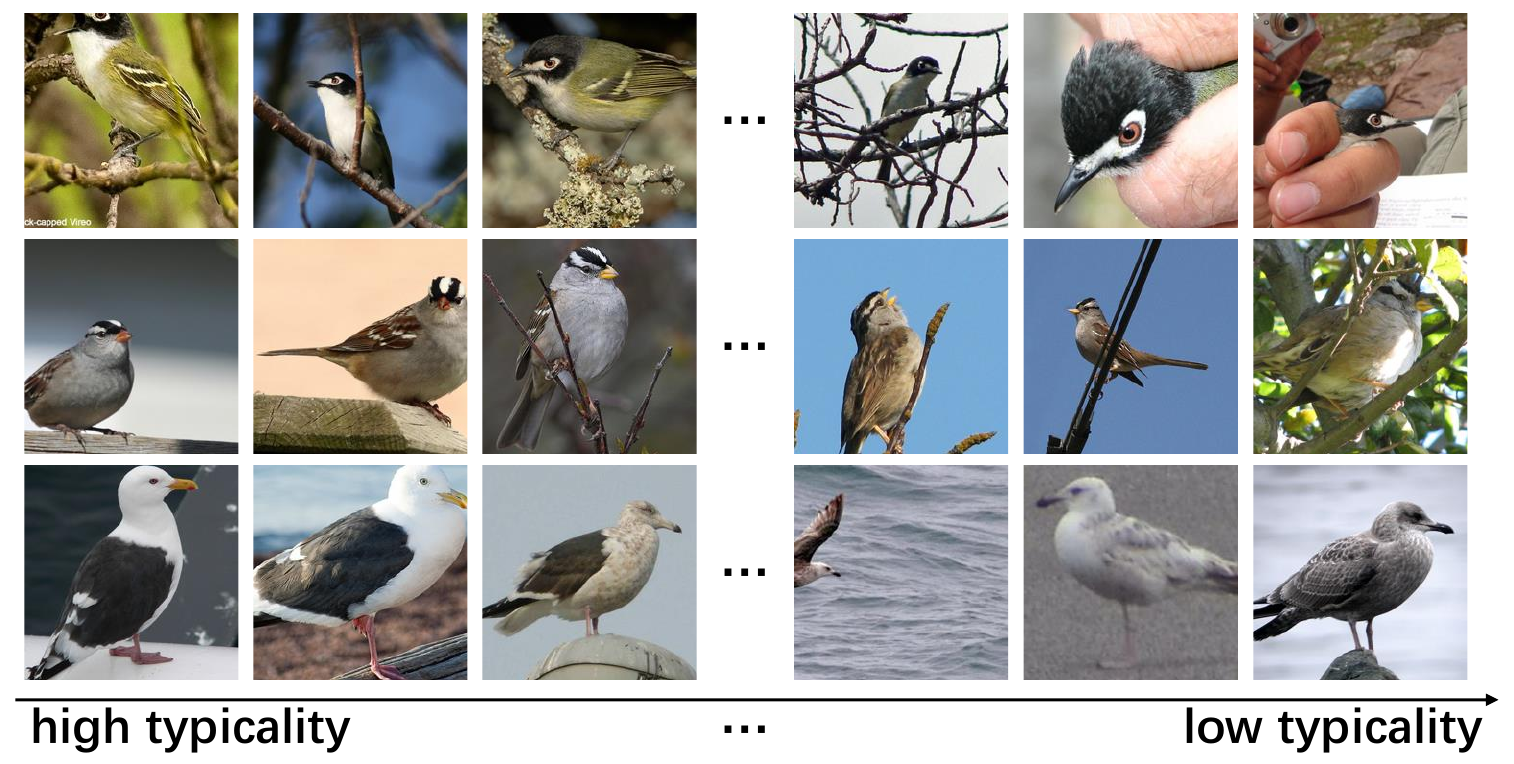}
    \vspace{-5pt}
    \caption{Sample retrieval on CUB. The three most typical and least typical are shown for each class.}
    \label{fig:typicality}
\end{figure}

\subsubsection{Sample Retrieval via \textit{Typicality}} \label{subsubsec:sample-retrieval}

Intuitively, the learnable prototypes of ProtoGCD capture the stereotype or template of each category, allowing us to explore an additional functionality: sample retrieval via \textit{typically}. We define \textit{typicality} as follows:
\begin{equation}
    \textit{typicality}(\mathbf{z}_i)=\boldsymbol{\mu}_{y_i}^\top\mathbf{z}_i,
    \label{eq:typicality}
\end{equation}
where $\boldsymbol{\mu}_{y_i}$ is the learned prototype of the $y_i$-th class. Based on Eq.~\eqref{eq:typicality}, we extract the most typical and least typical samples of \texttt{Black\_capped\_Vireo}, \texttt{White\_crowned\_Sparrow} and \texttt{Slaty\_backed\_Gull}, as depicted in Fig.~\ref{fig:typicality}. In the first row, three typical images (left) contain distinctive features, \eg, head and eyes. In contrast, indistinctive images (right) display that the vireo's body is partially obscured by human hands.

\section{Conclusion and Future Works} \label{sec:conclusion}
In this paper, we propose a novel framework called ProtoGCD to provide unified modeling and learn suitable representations for the task of generalized category discovery (GCD). ProtoGCD is characterized by its unification and unbiased features, as shown in Fig.~\ref{fig:two-charasteristics}. The \textbf{unification} is manifested on two levels: (1) Unified modeling of old and new classes (Fig.~\ref{fig:two-charasteristics} (a)). ProtoGCD employs joint prototypes and unified learning objectives for both old and new classes. (2) Task-level Unification (Fig.~\ref{fig:two-charasteristics} (b)). ProtoGCD could classify old classes, cluster new classes and detect unseen outliers, making it a unified classifier in the \textit{open-world}. Regarding the \textbf{unbiased} properties, there are also two dimensions: (1) ProtoGCD adopts a parametric classifier and DAPL, which aligns closely with the clustering objectives of GCD, together with two regularizations collectively learn suitable and less biased representations for GCD (Fig.~\ref{fig:two-charasteristics} (c)). (2) Our method flexibly assigns pseudo-labels to reduce the confirmation bias of incorrect pseudo-labels (Fig.~\ref{fig:two-charasteristics} (d)). In general, these two characteristics allow ProtoGCD to achieve balanced and remarkable performance for both old and new classes.
Besides, this paper introduces a novel method for estimating the number of new classes, considering both features and accuracy, enabling ProtoGCD to handle more realistic settings when the number of novel categories is unknown. Furthermore, we extend ProtoGCD to detect unseen categories, and achieve task-level unification. To validate the effectiveness of ProtoGCD, we conduct comprehensive experiments, including experiments on generic and fine-grained datasets, ablations and extended OOD detection. We also thoroughly analyze the advantages of ProtoGCD in broad scenarios, \eg, visualization of feature spaces and attention mechanisms, and corruption-shift cases. We further highlight the capability for typical sample retrieval.

ProtoGCD is an initial exploration oriented to handling scenarios involving various types of semantic-shift categories~\cite{zhu2024open,vaze2022openset}, including unlabeled novel categories and unseen outliers. We hope this work can inspire further research on versatile open-world classifiers and tackle more challenging settings, including filtering out outliers~\cite{zhang2022grow} in training data, continual category discovery~\cite{roy2022class,ma2024happy} requiring incrementally identifying novel categories while overcoming catastrophic forgetting. In both scenarios, OOD detection contributes to the discovery of new classes. Besides, future works could also design more suitable methods for long-tailed distributions in GCD and calibrate the confidence for both old and new classes. Beyond classification tasks, category and knowledge discovery can also be further applied to semantic segmentation~\cite{wu2024wps,liu2024mspe} and multimodal learning~\cite{guo2023dual,guo2024cross,guo2024crossmae,guo2025aligned}.


{\small
        \bibliographystyle{IEEEtran}
	\bibliography{refs}
}

\clearpage
\newpage

\section*{\Large Appendix}

\vspace{5pt}
\emph{Overview.}
This is the appendix for the paper entitled ``ProtoGCD: Unified and Unbiased Prototype Learning for Generalized Category Discovery''.
In the material, Section~\ref{sec:appendix-proof} provides the proof of Theorem~\ref{theorem:entropy} of the main text.
Section~\ref{sec:appendix-exp-details} presents the experimental details.
Section~\ref{sec:appendix-settings} demonstrates the relationship of several related task settings.
Section~\ref{sec:appendix-eval} elaborates on the evaluation metrics of GCD and OOD detection.
Section~\ref{sec:appendix-results} gives more detailed experimental results, including OOD detection, sensitivity analysis and visualizations.
An in-depth analysis of entropy regularization is included in Section~\ref{sec:appendix-entropy-herb}.
Section~\ref{sec:appendix-compare-simgcd} presents a comprehensive comparison between SimGCD and our method.
Finally, more discussion about the inter-class separation loss is provided in Section~\ref{sec:appendix-discuss-sep}.

\renewcommand\thesection{\Alph{section}}
\setcounter{section}{0}

\section{Proof of Theorem~2}\label{sec:appendix-proof}

\setcounter{theorem}{1}

\begin{theorem}
\label{theorem:appendix-entropy}
    Marginal entropy maximization $\mathcal{L}_\text{entropy}$ is equivalent to incorporating a prior distribution $\mathcal{U}$ across $K$ categories, where $\mathcal{U}$ is a uniform distribution.
\end{theorem}
\begin{proof}
    We firstly draw the Kullback–Leibler (KL) divergence between the marginal distribution $\overline{\mathbf{p}}$ and $\mathcal{U}$ as:
    \begin{equation}
        \text{KL}(\overline{\mathbf{p}}\Vert \mathcal{U})= \sum_{k=1}^K \overline{\mathbf{p}}^{(k)}\log\frac{\overline{\mathbf{p}}^{(k)}}{\mathcal{U}^{(k)}} =-H(\overline{\mathbf{p}})+\log K.
        \label{eq:appendix-kl-divergence}
    \end{equation}
    In Eq.~\eqref{eq:appendix-kl-divergence}, $\log K$ is a constant. Thus, maximizing the entropy is equivalent to minimizing the KL divergence between $\overline{\mathbf{p}}$ and $\mathcal{U}$, \ie, incorporating uniform distribution as a prior.
\end{proof}

\section{Experimental Details}\label{sec:appendix-exp-details}

\emph{Training Hyper-parameters.}
For a fair comparison, the basic training hyper-parameters follow prior methods~\cite{vaze2022gcd,fei2022xcon,pu2023dynamic}. We provide the list of basic training hyper-parameters in Table~\ref{tab:appendix-basic-params}, and the specific hyper-parameters of ProtoGCD are shown in Table~\ref{tab:appendix-protogcd-params}. Additionally, the temperature in \emph{prototype confidence} is $\tau_\text{sharp}$. And we set $\lambda_\text{entropy}=2$ for most datasets, while  $\lambda_\text{entropy}=1$ for CIFAR10~\cite{krizhevsky2009learning} and Aircraft~\cite{maji2013fine}.

\emph{Data Augmentations.}
Following the common practice of GCD~\cite{vaze2022gcd,fei2022xcon,pu2023dynamic}, we resize input images to $224\times 224$. We adopt conventional random augmentations for two views, including \texttt{RandomCrop}, \texttt{RandomHorizontalFlip} and \texttt{ColorJitter}.

\section{Relationship with Related Settings} \label{sec:appendix-settings}

We clarify the relationship between GCD and related fields. (1) \emph{Semi-Supervised Learning.} GCD extends SSL to the \textit{open-world}, where unlabeled data contain samples from new classes, while in SSL, labeled and unlabeled data share the same classes. (2) \emph{Unsupervised Clustering.} GCD could be viewed as deep transfer clustering~\cite{Han_2019_ICCV}. The underlying principle is to transfer the knowledge from labeled classes to cluster unlabeled novel categories. In contrast, without any prior knowledge, unsupervised clustering~\cite{xie2016unsupervised,caron2018deep} suffers from poor representation and ambiguity in the classification criterion. For example, models tend to face the dilemma of whether to group red flowers and red birds together or red flowers and blue flowers into the same cluster. In GCD, models grasp the prior knowledge and implicit cluster criterion in labeled data, as a result, models could obtain desired outcomes. (3) \emph{OOD Detection.} Both GCD and OOD detection consider open-set samples. OOD detection only needs to detect unseen samples, while GCD further requires the clustering of the new classes. (4) \emph{Novel Category Discovery.} GCD relaxes the assumption of NCD that unlabeled data exclusively come from novel classes. In GCD, unlabeled data contain samples from both old and novel classes. To conclude, GCD is a more challenging and pragmatic task.

\begin{table}[!tb]
\scriptsize
    \setlength{\tabcolsep}{6pt}
    \centering
    \renewcommand{\arraystretch}{0.8}
    \caption{Basic training hyper-parameters.}
    \vspace{-5pt}
    \label{tab:appendix-basic-params}
    \resizebox{.6\linewidth}{!}{
    \begin{tabular}{@{}cc@{}}
    \toprule
    Hyper-parameters & Value \\ \midrule
    train epochs & 200 \\
    batch size & 128 \\
    initial learning rate & 0.1 \\
    feature\_dim $d$ & 768 \\
    projection\_dim $d_h$ & 65,536 \\
    supervised weight $\lambda_\text{sup}$ & 0.35 \\ \bottomrule
    \end{tabular}
    }
\end{table}

\begin{table}[!tb]
    \setlength{\tabcolsep}{3pt}
    \centering
    \renewcommand{\arraystretch}{1}
    \caption{Specific hyper-parameters of ProtoGCD.}
    \vspace{-5pt}
    \label{tab:appendix-protogcd-params}
    \resizebox{.85\linewidth}{!}{
    \begin{tabular}{@{}ccc@{}}
    \toprule
    Params & Description & Value \\ \midrule
    $\lambda_\text{entropy}$ & weight of entropy regularization & 1 or 2 \\
    $\lambda_\text{sep}$ & weight of prototype separation & 0.1 \\
    $\tau_c$ & temperature of contrastive learning & 0.07 \\
    $\tau_\text{base}$ & temperature of predictions & 0.1 \\
    $\tau_\text{sharp}$ & temperature of sharpened soft labels & 0.05 \\
    $\tau_\text{sep}$ & temperature of prototype separation & 0.1 \\
    $e_\text{ramp}$ & ramp-up epochs & 100 \\ \bottomrule
    \end{tabular}
    }
\end{table}

\section{Evaluation Metrics} \label{sec:appendix-eval}

\subsection{Generalized Category Discovery}

GCD is essentially a clustering task, especially for novel classes. As described in the main text, during evaluation, we measure the clustering accuracy (ACC) of the model's predictions $\tilde y_i$ given the ground-truth labels $y_i$:
\begin{equation}
    ACC=\max_{p\in\Omega(\mathcal{Y}_u)}\frac{1}{M}\sum_{i=1}^M\mathds{1}\big\{y_i=p(\tilde y_i)\big\},
    \label{eq:appendix-acc}
\end{equation}
where $M=|\mathcal{D}_u|$ is the total number of unlabeled samples, and $\Omega(\mathcal{Y}_u)$ represents the set of all permutations that map the prediction to the ground-truth labels. We provide `All', `Old' and `New' accuracy for all data, data from ground-truth old classes, and data from ground-truth new classes, respectively. Eq.~\eqref{eq:appendix-acc} is achieved by the Hungarian algorithm. Note that we only perform Eq.~\eqref{eq:appendix-acc} \emph{once} on all the test data, and after acquiring $\Omega(\cdot)$, we then calculate `All', `Old' and `New' separately. This is canonical in GCD~\cite{vaze2022gcd,fei2022xcon,pu2023dynamic}.

\subsection{Out-of-Distribution Detection}

For Out-of-distribution (OOD) detection, we treat in-distribution (ID) samples as positives while OOD samples as negatives. In our experiments, the number ratio of ID to OOD samples is set to $1:1$.

\emph{FPR95.}
FPR95 is short for false positive rate at $95\%$ true positive rate. It could be interpreted as the probability that a negative sample (OOD) is misperceived as positive (ID) when $95\%$ of ID samples are correctly accepted, \ie, the true positive rate is $95\%$. 

\emph{AUROC.}
AUROC is short for Area Under the Receiver Operating Characteristic curve, which depicts the true positive rate (TPR) of ID against the false positive rate of OOD by varying the threshold. AUROC could be interpreted as the probability that we assign a higher OOD score to a positive sample than to a negative sample. AUROC is the threshold-independent metric.

\emph{AUPR.}
AUPR is the Area under the Precision-Recall curve, which shows the precision and recall against each other. AUPR-IN means that we treat ID as the positive. AUPR is also a threshold-independent metric.

\section{More Experimental Results} \label{sec:appendix-results}

\subsection{Evaluation of GCD without Prior Class Numbers}
We also conduct experiments in the scenarios without prior class numbers. Specifically, we train ProtoGCD with the estimated $\widetilde K$ (Table~6 in the main text) by \emph{Prototype Score}, as shown in Table~\ref{tab:unknown-number-performance}.

\begin{table}[!t]
    \setlength\tabcolsep{3pt}
    \centering
    \renewcommand{\arraystretch}{1.0}
    \caption{Performance of our method in the setting of unknown class numbers. Values in $()$ indicate the performance gap compared with known class number scenarios.}
    \vspace{-5pt}
    \label{tab:unknown-number-performance}
    \resizebox{.9\linewidth}{!}{
    \begin{tabular}{@{}ccccc@{}}
    \toprule
    Datasets & CIFAR-100 & ImageNet-100 & CUB & Scars \\ \midrule
    All & 81.9 $(0.0\downarrow)$ & 84.8 $\color{Red}(0.8\downarrow)$ & 61.4 $\color{Red}(1.8\downarrow)$ & 52.7 $\color{Red}(0.1\downarrow)$ \\
    Old & 82.9 $(0.0\downarrow)$ & 90.9 $\color{Red}(1.3\downarrow)$ & 66.2 $\color{Red}(2.3\downarrow)$ & 71.1 $\color{Red}(1.6\downarrow)$ \\
    New & 80.0 $(0.0\downarrow)$ & 81.8 $\color{Green}(1.9\uparrow)$ & 58.8 $\color{Red}(1.7\downarrow)$ & 43.8 $\color{Green}(0.6\uparrow)$ \\ \bottomrule
    \end{tabular}
    }
\end{table}

\subsection{Sensitivity of regularization weights}
To further explore the effects of the two regularization terms, we test sensitivity regarding their weights in Fig.~\ref{fig:regularization-weight}. For $\mathcal{L}_\text{entropy}$, the optimal value is $2.0$. Too large values hamper the learning of DAPL, leading to decreased performance. For $\mathcal{L}_\text{sep}$, the optimal value is $0.1$. Overall, $\mathcal{L}_\text{entropy}$ has a greater impact than $\mathcal{L}_\text{sep}$.

\begin{figure}[!t]
    \centering
    \includegraphics[width=.95\linewidth]{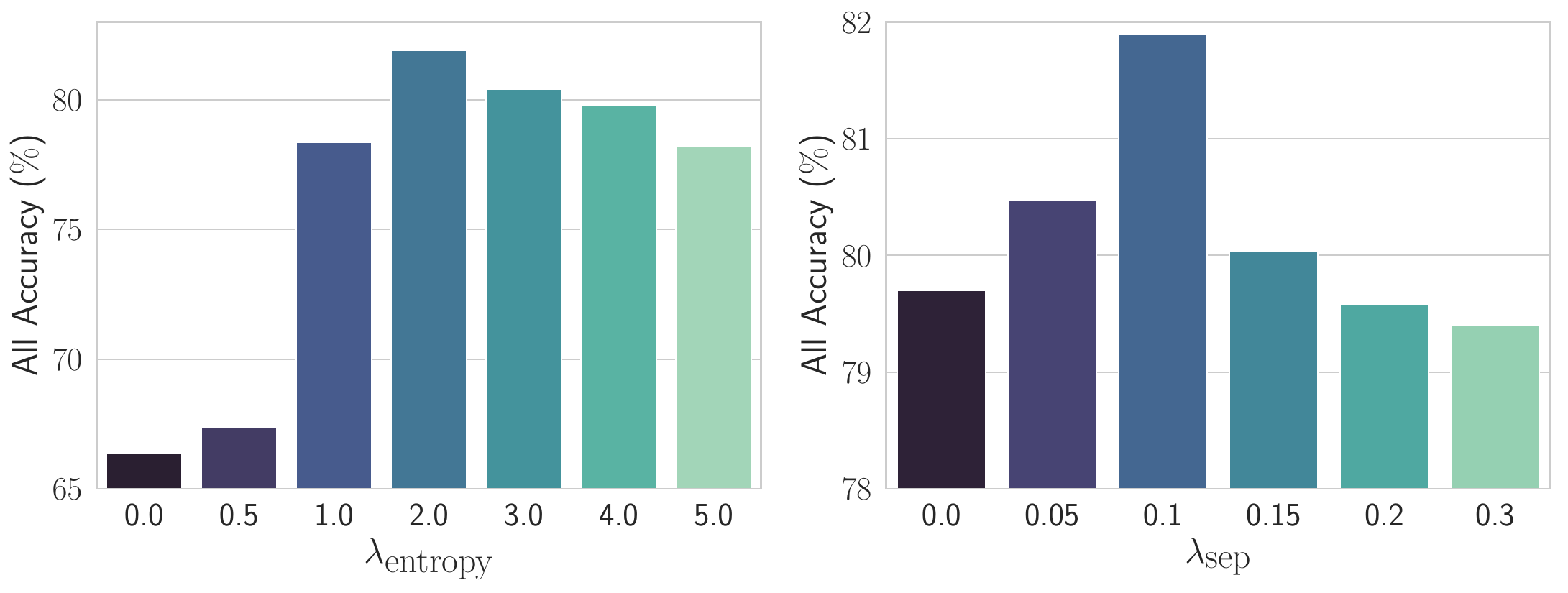}
    \vspace{-7pt}
    \caption{Performance over various weights of $\lambda_\text{entropy}$ and $\lambda_\text{sep}$.}
    \label{fig:regularization-weight}
\end{figure}

\begin{table*}[!t]
    \setlength{\tabcolsep}{4pt}
    \centering
    \renewcommand{\arraystretch}{1.0}
    \caption{OOD detection performance on different OOD datasets of CIFAR100 (in-distribution).}
    \vspace{-5pt}
    \label{tab:appendix-test-ood-cifar100}
    \resizebox{.9\textwidth}{!}{
    \begin{tabular}{@{}cccccccccc@{}}
    \toprule
    \multirow{2}{*}{$\mathcal{D}_\textrm{out}^\textrm{test}$} & \multicolumn{3}{c}{FPR95 $\downarrow$} & \multicolumn{3}{c}{AUROC $\uparrow$} & \multicolumn{3}{c}{AUPR-IN $\uparrow$} \\ \cmidrule(l){2-4} \cmidrule(l){5-7} \cmidrule(l){8-10} 
     & GCD & XCon & ProtoGCD & GCD & XCon & ProtoGCD & GCD & XCon & ProtoGCD \\ \midrule
    Texture & \textbf{29.92$\pm$1.12} & 42.81$\pm$0.71 & 31.31$\pm$0.91 & 92.99$\pm$0.23 & 90.74$\pm$0.24 & \textbf{93.90$\pm$0.19} & 98.44$\pm$0.05 & 97.95$\pm$0.06 & \textbf{98.73$\pm$0.04} \\
    SVHN & \textbf{47.80$\pm$0.93} & 51.54$\pm$0.72 & 50.65$\pm$1.05 & 90.54$\pm$0.22 & 90.89$\pm$0.17 & \textbf{91.21$\pm$0.21} & 98.01$\pm$0.06 & 98.20$\pm$0.04 & \textbf{98.23$\pm$0.05} \\
    Places365 & \textbf{49.36$\pm$1.28} & 69.20$\pm$0.54 & 56.17$\pm$0.75 & \textbf{86.68$\pm$0.48} & 81.31$\pm$0.41 & 84.20$\pm$0.36 & \textbf{96.76$\pm$0.14} & 95.41$\pm$0.15 & 96.26$\pm$0.09 \\
    TinyImageNet & 59.08$\pm$1.26 & 60.88$\pm$1.30 & \textbf{58.93$\pm$1.00} & 84.62$\pm$0.42 & 84.00$\pm$0.35 & \textbf{85.94$\pm$0.33} & 96.41$\pm$0.14 & 96.26$\pm$0.08 & \textbf{96.81$\pm$0.08} \\
    LSUN & 71.16$\pm$0.92 & 63.40$\pm$0.81 & \textbf{60.89$\pm$1.21} & 83.42$\pm$0.20 & 84.68$\pm$0.35 & \textbf{87.07$\pm$0.21} & 96.41$\pm$0.04 & 96.61$\pm$0.10 & \textbf{97.16$\pm$0.06} \\
    iSUN & 69.15$\pm$0.69 & 65.02$\pm$0.99 & \textbf{64.03$\pm$1.10} & 82.87$\pm$0.27 & 83.97$\pm$0.32 & \textbf{84.51$\pm$0.45} & 96.15$\pm$0.07 & \textbf{96.44$\pm$0.09} & 96.42$\pm$0.12 \\
    CIFAR10 & 71.97$\pm$0.73 & 68.53$\pm$1.18 & \textbf{63.53$\pm$1.12} & 77.59$\pm$0.32 & 78.13$\pm$0.45 & \textbf{80.18$\pm$0.33} & 94.47$\pm$0.09 & 94.58$\pm$0.13 & \textbf{95.03$\pm$0.10} \\ \midrule
    Mean & 56.92 & 60.20 & \textbf{55.07} & 85.53 & 84.82 & \textbf{86.72} & 96.66 & 96.49 & \textbf{96.95} \\ \bottomrule
    \end{tabular}
    }
\end{table*}

\begin{table*}[!t]
    \setlength{\tabcolsep}{4pt}
    \centering
    \renewcommand{\arraystretch}{1.0}
    \caption{OOD detection performance on different OOD datasets of ImageNet-100 (in-distribution).}
    \vspace{-5pt}
    \label{tab:appendix-test-ood-imagenet}
    \resizebox{.9\textwidth}{!}{
    \begin{tabular}{@{}cccccccccc@{}}
    \toprule
    \multirow{2}{*}{$\mathcal{D}_\textrm{out}^\textrm{test}$} & \multicolumn{3}{c}{FPR95 $\downarrow$} & \multicolumn{3}{c}{AUROC $\uparrow$} & \multicolumn{3}{c}{AUPR-IN $\uparrow$} \\ \cmidrule(l){2-4} \cmidrule(l){5-7} \cmidrule(l){8-10} 
     & GCD & XCon & ProtoGCD & GCD & XCon & ProtoGCD & GCD & XCon & ProtoGCD \\ \midrule
    Texture & 46.62$\pm$1.27 & 39.79$\pm$0.98 & \textbf{21.75$\pm$1.64} & 91.60$\pm$0.20 & 93.70$\pm$0.21 & \textbf{94.60$\pm$0.38} & 98.37$\pm$0.05 & 98.61$\pm$0.05 & \textbf{98.75$\pm$0.14} \\
    Places365 & 66.37$\pm$1.60 & 67.82$\pm$1.54 & \textbf{56.47$\pm$1.65} & \textbf{87.00$\pm$0.43} & 86.88$\pm$0.29 & 85.09$\pm$0.46 & \textbf{97.50$\pm$0.11} & 97.48$\pm$0.06 & 96.67$\pm$0.13 \\
    iNaturalist & 70.30$\pm$1.45 & 69.87$\pm$1.57 & \textbf{52.29$\pm$1.28} & 86.28$\pm$0.49 & 86.29$\pm$0.31 & \textbf{87.72$\pm$0.57} & 97.28$\pm$0.11 & 97.20$\pm$0.06 & \textbf{97.31$\pm$0.16} \\
    ImageNet-O & 63.47$\pm$0.90 & 61.70$\pm$0.76 & \textbf{48.91$\pm$0.98} & 85.75$\pm$0.39 & 87.23$\pm$0.40 & \textbf{87.89$\pm$0.46} & 97.02$\pm$0.12 & 97.10$\pm$0.11 & \textbf{97.27$\pm$0.10} \\
    OpenImage-O & 64.34$\pm$1.32 & 60.84$\pm$1.02 & \textbf{46.64$\pm$1.10} & 86.56$\pm$0.45 & 88.43$\pm$0.27 & \textbf{89.45$\pm$0.55} & 97.35$\pm$0.12 & 97.17$\pm$0.06 & \textbf{97.59$\pm$0.20} \\ \midrule
    Mean & 62.22 & 60.00 & \textbf{45.21} & 87.44 & 88.51 & \textbf{88.95} & 97.50 & 97.51 & \textbf{97.52} \\ \bottomrule
    \end{tabular}
    }
\end{table*}

\subsection{Detailed OOD Experimental Results}
We provide detailed OOD results in Table~\ref{tab:appendix-test-ood-cifar100} and Table~\ref{tab:appendix-test-ood-imagenet}, including the standard derivation and the AUPR-IN metric. As Table~\ref{tab:appendix-test-ood-cifar100} shows, ProtoGCD consistently outperforms other counterparts for OOD detection.

\subsection{More Visualization Results}
In the main text, we provide the feature visualization of the CIFAR10 dataset. Here, we further visualize the features of the CUB dataset with more classes. ProtoGCD could obtain feature representations with improved intra-class compactness and inter-class separation. In Fig.~\ref{fig:vis-tsne-cub}, the clusters of \texttt{Long\_tailed\_Jaeger}, \texttt{Tennessee\_Warbler} and \texttt{Loggerhead\_Shrike} in GCD are dispersed, while in our methods are more compact. Besides, classes like \texttt{White\_crowned\_Sparrow}, \texttt{Tree\_Sparrow} and \texttt{European\_Goldfinch} in GCD~\cite{vaze2022gcd} tend to overlap and become intertwined, posing challenges to distinguish among them, while our method achieves clear cluster boundaries and separated clusters.

\begin{figure*}[!th]
    \centering
    \includegraphics[width=.98\textwidth]{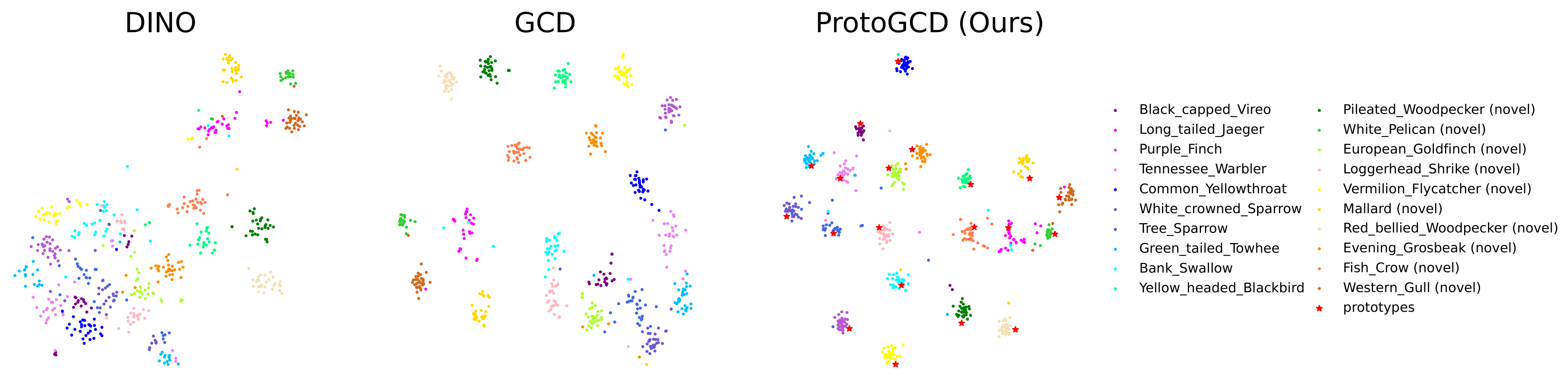}
    \vspace{-5pt}
    \caption{Visualizations of the feature space on CUB. Features of old classes are depicted in cool colors (\eg, $\color{Orchid}{\bullet}$, $\color{blue}{\bullet}$, $\color{cyan}{\bullet}$, $\color{SeaGreen}{\bullet}$) while novel categories in warm colors (\eg, $\color{Yellow}{\bullet}$, $\color{orange}{\bullet}$, $\color{pink}{\bullet}$, $\color{YellowGreen}{\bullet}$). Additionally, the learnable prototypes are denoted as ${\color{red}{\star}}$. Our method provides improved inter-class separation and intra-class compactness.}
    \label{fig:vis-tsne-cub}
\end{figure*}

\section{In-depth Discussion of Entropy Regularization on Herb} \label{sec:appendix-entropy-herb}
Although entropy regularization $\mathcal{L}_\text{entropy}$ has implicitly imposed the assumption of a uniform distribution on the dataset, which might conflict with the long-tailed distributions for Herb. We have conducted a detailed sensitivity analysis of $\mathcal{L}_\text{entropy}$ on the Herb dataset, as in Fig.~9 of the main text. Overall, despite the Herb dataset being a long-tailed dataset, the results indicate a huge degradation in the absence of marginal entropy maximization $\mathcal{L}_\text{entropy}$ (44.5\% $\to$ 29.4\%, as shown in `0.0' of Fig.~9). Even when imposing a small weight, \eg, 0.1, there is a notable performance enhancement (29.4\% $\to$ 36.2\%). In summary, the most suitable weight $\lambda_\text{entropy}$ is approximately 2.0. From the experimental results, we argue that $\mathcal{L}_\text{entropy}$ is still a relatively applicable regularization in GCD. Some explanations are discussed as follows:

\begin{itemize}
    \item $\mathcal{L}_\text{entropy}$ is a soft regularization rather than the hard constraint. It is noteworthy that $\mathcal{L}_\text{entropy}$ is essentially different from the hard constraint, \eg, UNO~\cite{Fini_2021_ICCV} that rigidly follows equipartition constraints via the Sinkhorn-Knopp algorithm~\cite{cuturi2013sinkhorn}. The hard constraint could drastically damage the result. For example, UNO has a very weak performance on Herb in Table~3 of the main paper. In comparison, by incorporating $\mathcal{L}_\text{entropy}$ as a differential part of the overall learning objective, we could adjust the weight $\lambda_\text{entropy}$ to balance its influence. If $\mathcal{L}_\text{entropy}$ is completely discarded via $\lambda_\text{entropy}=0$, the model could be restricted to trivial solutions, leading to significant performance degradation. Conversely, if $\lambda_\text{entropy}$ is too large, it contradicts the long-tailed distribution of Herb. As a result, we could choose a proper $\lambda_\text{entropy}$ to obtain desirable results, for example, $\sim 2.0$ in Fig.~9.
    \item For old and new classes, there is a gap between the model's marginal probabilities $\overline{\mathbf{p}}$ and the model's predicted classes $\hat y$. Formally, let $\overline p^\text{old}=\sum_{c\in\mathcal{C}_{old}}\overline{\mathbf{p}}^{(c)}$ and $\overline p^\text{new}=\sum_{c\in\mathcal{C}_{new}}\overline{\mathbf{p}}^{(c)}$ denote the predicted probabilities for old and new classes, both are scalars and $\overline p^\text{old}+\overline p^\text{new}=1$. Then let $r^\text{old}$ and $r^\text{new}$ denote the proportions of samples that the model classified as old and new classes, \ie, $r^\text{old}=\frac{1}{N}\sum_{i=1}^N 1(\hat y_i<=K^\text{old}),r^\text{new}=\frac{1}{N}\sum_{i=1}^N 1(\hat y_i>K^\text{old})$ and $r^\text{old}+r^\text{new}=1$. Here $\hat y_i=\arg\max_k p(y=k|\mathbf{z}_i)$ denotes the predicted class of the $i$-th sample. Due to the confidence gap between old and new classes, the model generally exhibits higher confidence in old classes (because old classes are partially labeled while new classes are fully unlabeled). Consequently, there exists a disparity between $r^\text{old}$ and $\overline p^\text{old}$, so as to $r^\text{new}$ and $\overline p^\text{new}$. \textbf{The entropy regularization $\mathcal{L}_\text{entropy}$ is directly applied to $\overline p^\text{old}$ and $\overline p^\text{new}$, while the actual long-tailed distribution is associated with $r^\text{old}$ and $r^\text{new}$.} To sum up, considering the confidence gap between old and new classes and the weak confidence calibration performance in GCD, employing a maximum entropy constraint remains a relatively suitable approach. Similar findings have been reported in a recent work~\cite{ma2024happy}. We believe that addressing the gap between old and new classes and reducing the disparity between $\overline{p}$ and $r$ will be a valuable open problem in GCD.
\end{itemize}

\begin{table*}[!t]
    \setlength{\tabcolsep}{25pt}
    \centering
    \renewcommand{\arraystretch}{1.38}
    \caption{The summarized differences between \simgcd{SimGCD} and \protogcd{ProtoGCD} from various perspectives.}
    \vspace{-5pt}
    \label{tab:diff-simgcd-protogcd}
    \resizebox{.85\textwidth}{!}{
    \begin{tabular}{@{}ccc@{}}
\toprule
\textbf{Perspectives} & \simgcd{SimGCD} & \multicolumn{1}{c}{\protogcd{ProtoGCD}} \\ \midrule
Modeling & Purely Discriminative & Hybrid = Generative + Discriminative \\ \hline
Prototypes & $\ell_2$-normed Classifier & Class-wise Distribution \\ \hline
Pseudo-Labeling & Self-distillation & Dual-level Adaptive Pseudo-Labeling \\ \hline
Regularization & Entropy Maximization & Entropy Maximization + Inter-class Separation \\ \hline
Extensions & N/A & Class Number Estimation + OOD Detection \\ \bottomrule
\end{tabular}
    }
\end{table*}

\section{Detailed Comparison with SimGCD} \label{sec:appendix-compare-simgcd}
SimGCD~\cite{wen2023parametric} is a recent parametric-based GCD method. Here, we provide a comprehensive comparison between SimGCD and our ProtoGCD.

\noindent{\textbf{(a) Differences in the model structure design.}}

\begin{itemize}
    \item \textbf{About prototypical classifier.} Although both SimGCD and ProtoGCD utilize prototypes, there is a significant distinction in the meaning of the term `prototype'. \simgcd{SimGCD} refers to its classifier as a prototypical classifier merely due to implementing $\ell_2$ normalization and omitting the bias term upon conventional classifier. So there is no fundamental difference from the traditional classifier. Overall, \simgcd{SimGCD} can be regarded as a purely \textbf{discriminative} model. By contrast, the prototypes in our \protogcd{ProtoGCD} represent the class-wise probability distributions (Eq.~(1) in the main text), \ie, von Mises–Fisher (vMF) distribution~\cite{mardia2000directional}. It is a form of generative modeling. Then, we derive the posterior predictive probabilities in Eq.~(2). The learning mechanism incorporates both discriminative learning with pseudo-labels and generative learning with inter-class separation loss $\mathcal{L}_\text{sep}$ (in Eq.~(14)) and \emph{prototype confidence}. Overall, \protogcd{ProtoGCD} is a \textbf{hybrid} model that combines both generative and discriminative modeling.
    \item Moreover, contrastive learning~\cite{pmlr-v119-chen20j,NEURIPS2020_d89a66c7} is a versatile technique that has been widely adopted in the literature of GCD~\cite{vaze2022gcd,pu2023dynamic,Zhao_2023_ICCV,wen2023parametric}, which helps ensure basic feature representations, so we follow their common practice in our method.
\end{itemize}

\noindent\textbf{(b) Differences in the loss function design.}
\begin{itemize}
    \item \textbf{About regularization terms.} \protogcd{ProtoGCD} primarily comprises two regularization terms, \ie, marginal entropy maximization $\mathcal{L}_\text{entropy}$ and inter-class (prototype) separation regularization $\mathcal{L}_\text{sep}$. Here, $\mathcal{L}_\text{sep}$ is our main novelty. Similar to contrastive learning, entropy regularization $\mathcal{L}_\text{entropy}$ is also commonly employed in the literature of GCD~\cite{wen2023parametric,NEURIPS2023_3f52ab43}, which helps to alleviate trivial solutions in clustering. However, many previous methods, including \simgcd{SimGCD}, rely solely on entropy maximization as a constraint, and neglect the constraints within the feature space, resulting in less separable clusters. To overcome this issue, we propose to explicitly decrease inter-class overlapping via the separation regularization $\mathcal{L}_\text{sep}$. In this way, \protogcd{ProtoGCD} could obtain more suitable representations for GCD and remarkable accuracy for both old and new classes. Besides, the prototype separation loss $\mathcal{L}_\text{sep}$ aligns with our generative modeling, which helps reduce the overlap between distributions of different classes and makes them more separable.
    \item \textbf{About the pseudo-labeling mechanism.} The cross-view prediction is a general framework, while the design of pseudo-labels within this framework is of vital importance. In this regard, our \protogcd{ProtoGCD} have significant differences from \simgcd{SimGCD}. Specifically, \simgcd{SimGCD} simply employs the off-the-shelf self-distillation borrowed from DINO~\cite{caron2021emerging}, which fails to consider the specific characteristics of GCD, resulting in suboptimal performance. In this task, there is an inherent imbalance in labeling conditions between old and new classes, leading to an obvious confidence gap among samples. As a result, the informativeness for pseudo-labeling varies remarkably among different samples. Besides, at early training stages, the model's capabilities are relatively weak and could bring larger noise to pseudo-labels compared with later training stages, so the optimal configuration for pseudo-labels is continuously evolving. These issues motivate us to propose dual-level adaptive pseudo-labeling (DAPL) in \protogcd{ProtoGCD}. Our method is specifically designed to consider varying confidence levels among samples and varying model capabilities across learning stages, and could effectively mitigate biases while achieving efficient self-learning.
\end{itemize}

\noindent\textbf{(c) Other contribution of \protogcd{ProtoGCD}.}
\simgcd{SimGCD} solely focuses on the GCD task. In contrast, we provide theoretical analysis for our \protogcd{ProtoGCD}, and we further devise a method to estimate the number of classes and extend \protogcd{ProtoGCD} to detect OOD samples. 

To conclude, we summarize the differences between these two methods in Table~\ref{tab:diff-simgcd-protogcd}.

\section{More Discussion about Inter-Class Separation Regularization} \label{sec:appendix-discuss-sep}
Although the dispersion loss $\mathcal{L}(m,n)$ (Eq.~(6) in the DCCL paper~\cite{pu2023dynamic}) and our inter-class separation loss $\mathcal{L}_\text{sep}$ (Eq.~(14) in our paper) share a similar goal, our approach is generally more efficient and stable. Specifically, DCCL~\cite{pu2023dynamic} is a non-parametric method following the EM-like framework, where the class-wise conception representations (analogous to prototypes in ProtoGCD) are non-learnable and updated via the exponential moving average (EMA). In each iteration, DCCL requires sampling multiple instances for each conception label, averaging their features, and subsequently computing the dispersion loss. This sampling and averaging process is inefficient, and if the number of samples per class is insufficient, it may lack representativeness, leading to instability. Additionally, DCCL relies on the threshold $\tau^M$ to filter the conception pairs with high uncertainty. Tuning this hyper-parameter might increase the experimental burden. By contrast, our method directly applies a separation loss to learnable prototypes $\{\boldsymbol{\mu}_c\}_{c=1}^K$ (see $\mathcal{L}_\text{sep}$ in Eq.~(14) of our paper), which requires no sampling and averaging process and is computationally simple. $\mathcal{L}_\text{sep}$ enables end-to-end training, making it highly efficient. Furthermore, the learnable prototypes in our method effectively represent each class, eliminating issues about insufficient representation due to limited samples, thereby ensuring the stability of ProtoGCD.

\end{document}